\documentclass[11pt]{article}

\usepackage{amsmath,epsfig}
\usepackage{amssymb}
\usepackage{algorithm}
\usepackage{algorithmic}
\usepackage{subfigure}
\usepackage{amsthm}

\newcommand{\tr}{\ensuremath{ \mathrm{tr}}}
\newcommand{\numsamp}{\ensuremath{N}}
\newcommand{\pmes}{\ensuremath{\nu}}
\newcommand{\Demb}{\ensuremath{A}}
\newcommand{\mar}{\ensuremath{\gamma}}
\newcommand{\minm}{\ensuremath{\eta_{m,\delta}}}
\newcommand{\R}{\ensuremath{\mathbb{R}}}
\newcommand{\class}{\ensuremath{C}}
\newcommand{\Knb}{\ensuremath{Q}}
\newcommand{\M}{\ensuremath{\mathcal{M}}}
\newcommand{\half}{\ensuremath{\frac{1}{2}}}

\newtheorem{theorem}{Theorem}
\newtheorem{definition}{Definition}
\newtheorem{proposition}{Proposition}

\begin{document}


\title{Nonlinear Supervised Dimensionality Reduction via Smooth Regular Embeddings}


%
%

\author{Cem \"Ornek  and Elif Vural \\      
       Department of Electrical and Electronics Engineering, METU, Ankara
        }

\date{}

\maketitle

\begin{abstract}
The recovery of the intrinsic geometric structures of data collections is an important problem in data analysis. Supervised extensions of several manifold learning approaches have been proposed in the recent years. Meanwhile, existing methods primarily focus on the embedding of the training data, and the generalization of the embedding to initially unseen test data is rather ignored. In this work, we build on recent theoretical results on the generalization performance of supervised manifold learning algorithms. Motivated by these performance bounds, we propose a supervised manifold learning method that computes a nonlinear embedding while constructing a smooth and regular interpolation function that extends the embedding to the whole data space in order to achieve satisfactory generalization. The embedding and the interpolator are jointly learnt such that the Lipschitz regularity of the interpolator is imposed while ensuring the separation between different classes. Experimental results on several image data sets show that the proposed method outperforms  traditional classifiers and the supervised dimensionality reduction algorithms in comparison in terms of classification accuracy in most settings.


\end{abstract}



\section{Introduction}
\label{sec:intro}

In many data analysis applications, collections of data are acquired in a high-dimensional ambient space; however, the intrinsic dimension of data is much lower. For instance, the face images of a person reside in a high-dimensional space, however, they are concentrated around a low-dimensional manifold that can be parameterized with a few variables such as pose and illumination parameters. An important problem of interest in data analysis has been the learning of low-dimensional models that provide suitable representations of data for accurate classification. Many supervised manifold learning methods have been proposed in the recent years that aim to enhance the separation between training samples from different classes while respecting the geometric structure of data manifolds. However, the generalization capabilities of such methods to initially unavailable novel samples have rather been overlooked so far. In this work, we propose a nonlinear supervised dimensionality reduction method that builds on theoretically established generalization bounds for manifold learning.

Classical methods such as LDA and Fisher's linear discriminant reduce the dimensionality of data by learning a projection so that the between-class separation is increased while the within-class separation is reduced. In the recent years, much research effort has focused on the discovery of low-dimensional structures in data sets, which gave rise to the topic of manifold learning \cite{Tenenbaum00}, \cite{Roweis00}, \cite{Belkin03},  \cite{He04}, \cite{Donoho03}, \cite{Zhang2005}. Following these works, many supervised extensions of methods such as the Laplacian eigenmaps algorithm \cite{Belkin03} have been proposed. Linear dimensionality reduction methods such as  \cite{Sugiyama07}, \cite{Hua12}, \cite{Yang11}, \cite{Zhang12}, \cite{Li13}, \cite{Cui12}, \cite{Wang09}, \cite{YuSZH16} learn a linear projection of training samples onto a lower-dimensional domain, where  the distance between samples from different classes are increased and the distances within the same class are decreased. Most of these methods include a structure preservation objective as well, which aims to map nearby samples in the original domain to nearby locations in the new domain of embedding. Nonlinear methods such as \cite{Raducanu12} pursue a similar objective in the learnt embedding; however, the embedding is given by a pointwise nonlinear mapping instead of a linear projection.

The performance of linear methods depends largely on the distribution of the data in the original ambient space, since the distribution of the data after the embedding is strictly dependent on the original distribution via a linear projection. Nonlinear dimensionality reduction methods such as \cite{Raducanu12} have greater flexibility in the learnt representation. However, two critical issues arise concerning supervised dimensionality reduction methods: First, most nonlinear methods compute a pointwise mapping only for the initially available data samples. In order to generalize them to initially unavailable points, an interpolation needs to be done, which is called the out-of-sample extension of the embedding. Second, existing dimensionality reduction methods focus on the properties of the computed embedding only as far as the training samples are concerned: Existing algorithms mostly aim to increase the between-class separation and preserve the local structure, however, only for the training data. Meanwhile, the important question is how well these algorithms generalize to test data. This question is even more critical for nonlinear dimensionality reduction methods, as the classification performance obtained on test data will not only depend on the properties of the embedding of the training data, but also on the properties of the interpolator used for extending the embedding to the whole ambient space. Several methods have been proposed to solve the out-of-sample extension problem, such as unsupervised generalizations with smooth functions \cite{Bengio04}, \cite{QiaoZWZ13}, \cite{ChenWG13}, \cite{PeherstorferPB11} or semi-supervised interpolators \cite{VuralG16}. These methods intend to generalize an already computed embedding to new data and are constrained by the initially prescribed coordinates for training data. Meanwhile, the best strategy for achieving satisfactory generalization to test data would not consist in learning the embedding and the interpolation sequentially, but rather in learning them in a joint and coherent manner.

In this work, we propose a nonlinear supervised manifold learning method for classification where the embeddings of training data are learned and optimized in a joint way along with the interpolator that extends the embedding to the whole ambient space. A distinctive property of our method is the fact that it explicitly aims to have good generalization to test data in the learning objective. In order to achieve this, we build on the previous work \cite{VuralGTR} where a theoretical analysis of supervised manifold learning is proposed. The theoretical results in \cite{VuralGTR} show that for good classification performance,  the separation between different classes in the embedding of training data needs to be sufficiently high, while at the same time the interpolation function that extends the embedding to test data must be sufficiently regular. For good generalization to initially unavailable test samples, a compromise needs to be found between these two important criteria. In this work, we adopt radial basis function interpolators for the generalization of the embedding, and learn the embedding of the training data and the parameters of the interpolator, i.e., the coefficients and the scale parameter of the interpolation function, at the same time with a joint optimization algorithm. The analysis in  \cite{VuralGTR} characterizes the regularity of an interpolator via its Lipschitz regularity. We first derive an upper bound on the Lipschitz constant of the interpolator in terms of the parameters of the embedding. Then, relying on the theoretical analysis in \cite{VuralGTR}, we propose to optimize an objective function that maximizes the separation between different classes and preserves the local geometry of training samples, while at the same time minimizing an upper bound on the Lipschitz constant of the RBF interpolator. We propose an alternating iterative optimization scheme that first updates the embedding coordinates, and then the interpolator parameters in each iteration. We test the classification performance of the proposed method on several real data sets and show that it outperforms the supervised manifold learning methods in comparison and traditional classifiers.

The contributions of this paper with respect to previous works are the following:

\begin{itemize}

\item The generalization capability of the classifier resulting from a nonlinear supervised embedding is considered during the learning of the embedding for the first time.

\item An embedding along with a continuous interpolator is learnt with an optimization objective based on recent theoretical results on the performance of supervised manifold learning methods.

\item We show that enforcing the Lipschitz regularity of the interpolator function in addition to the separation between the different classes improves the accuracy of the classifier in most experimental settings.

\end{itemize}

The rest of the paper is organized as follows. In Section \ref{sec:related_work}, we overview the related work. In Section \ref{ssec:theo_persp}, we review the recent theoretical results that motivate our method and in Section \ref{sec:proposed_method}, we formulate the supervised manifold learning problem and present the proposed algorithm. In Section \ref{sec:exp_results}, we evaluate our method with experiments on several face and object data sets. Finally, we conclude in Section \ref{sec:conclusion}.

\section{Related Work}
\label{sec:related_work}

\subsection{Unsupervised manifold learning}

Manifold learning algorithms aim to compute a low-dimensional representation of data that is coherent with its intrinsic geometry, which is characterized in several different ways via geodesic distances \cite{Tenenbaum00}, locally linear representations \cite{Roweis00}, second order characteristics  \cite{Donoho03}, and graph spectral decompositions \cite{Belkin03}, \cite{He04} in previous works. When the underlying manifold model is not analytically known, it is common to represent data with a graph model. Given a set of data samples $X=\{x_i\}_{i=1}^N\subset \mathbb R^n$, most manifold learning methods build a data graph such that two samples $x_i$ and $x_j$ are linked with an edge when they are nearest neighbors of each other ($x_i \sim x_j$). The edge weights $w_{ij}$ are typically assigned with respect to a similarity measure between neighboring samples. 

Denoting as $W$ the weight matrix  containing the edge weights $w_{ij}$, and the degree matrix $D$ as the diagonal matrix having as $i$-th entry the total edge weight $d(i)=\sum_{x_j \sim x_i} w_{ij}$  between $x_i$ and its neighbors, the graph Laplacian matrix is given by $L=D-W$. The Laplacian eigenmaps algorithm \cite{Belkin03} maps each data sample $x_i \in \mathbb R^n$ to a sample $y_i \in \mathbb R^d$ such that the following optimization problem is solved
\begin{equation}
\min_Y \tr(Y^T L \, Y) = \min_Y \sum_{i \sim j} \| y_i - y_j \|^2 w_{ij}, \text{ s.t. } Y^T Y = I
\end{equation}
where $Y=[y_1 \ y_2 \ \dots \ y_N]^T$ is the data matrix consisting of the coordinates to be learned and $I$ is the identity matrix. Hence, the Laplacian eigenmaps algorithm formulates the new coordinates of data as the functions that have the slowest variation on the data graph, so that neighboring samples in the original domain are mapped to nearby coordinates in the new domain of embedding. The locality preserving projections (LPP) \cite{He04} algorithm has the same objective; however, the new coordinates $y_i=P^T x_i$ are constrained to be given by a linear projection of the original coordinates.

\subsection{Supervised manifold learning}

Many supervised manifold learning algorithms have been proposed in the recent years, most of which are extensions of the Laplacian eigenmaps method. These methods seek to embed data into new coordinates such that neighboring samples in the same class are mapped to nearby coordinates, while samples from different classes are mapped to distant points. This is often represented as an objective function that minimizes $\tr(Y^T L_w \, Y)$ while maximizing $\tr(Y^T L_b \, Y) $, where $L_w$ and $L_b$ are the within-class and between-class Laplacian matrices, derived respectively from the within-class and between-class weight matrices $W_w$ and $W_b$. The between-class edges in $W_b$ can be set with respect to different strategies in different methods. The supervised dimensionality reduction problem is formulated in \cite{Raducanu12} as
\begin{equation} 
\label{eq:supLap_formal}
\min_{Y} \tr(Y^T L_w Y) - \mu \,  \tr(Y^T L_b Y) \text{ subject to } Y^T Y = I
\end{equation}
where $\mu>0$ is a constant that adjusts the weight between the structure preservation and the class-aware discrimination terms. Similar formulations are adopted in \cite{LDE},  \cite{Hua12}; however, under the linear projection constraint $y_i=P^T x_i$.  The recent work in \cite{MaronidisTP15} is based on a similar objective, which also exploits subclass information by identifying favorable data connections within the same class. A local adaptation of the Fisher discriminant analysis is proposed in \cite{Sugiyama07}. A projection matrix $P$ is sought so that the following objective is maximized
\begin{equation}
\arg \max_P \tr( (P^T S_w P)^{-1} P^T S_b P )
\end{equation}
where $S_w$ and $S_b$ are the within-class and between-class scatter matrices obtained with  the edge weights of samples on the data graph. The methods in \cite{Yang11}, \cite{Zhang12}, \cite{Cui12}, \cite{Wang09}, \cite{GaoMZGL13} also optimize a similar Fisher-like objective by maximizing the between-class local scatter and minimizing the within-class local scatter. The  method in \cite{LFDP} proposes a scatter discriminant analysis to learn embeddings of local image descriptors. In another recent work \cite{ChenWLL17}, the optimization of within- and between-class local scatters is formulated via $\ell_1$-norms for robustness against image degradations.

Several supervised linear dimensionality reduction methods are based on preserving locally linear representations of data. The algorithm in \cite{ESLLE} provides a supervised extension of the well-known LLE method \cite{Roweis00} by introducing a label-dependent distance function; however, it is a nonlinear method without an explicit consideration of the out-of-sample problem.  The Neighborhood Preserving Discriminant Embedding method presented in \cite{NPDE} is a linear dimensionality reduction method extending the unsupervised NPE method \cite{HeCYZ05} based on locally linear representations. The Hybrid Manifold Embedding method \cite{HyME} computes a locally linear but globally  nonlinear mapping function by first grouping the data into local subsets via geodesic clustering and then learning a supervised embedding of each cluster. The supervised dimensionality reduction method in \cite{ZhouS17} partitions the manifold into local regions and takes into account the variation of the embedding along tangent directions of the manifold.

\subsection{Continuous embeddings via nonlinear functions}

The vast majority of supervised dimensionality reduction methods relies on linear projections, and the methods computing a continuous supervised nonlinear embedding are less common. The generalization of the embedding of a given set of training samples to the whole space via continuous interpolation functions is known as the out-of-sample extension problem. The out-of-sample problem is of critical importance especially for nonlinear supervised manifold learning methods computing a pointwise embedding only at training samples. 

The Nystr\"om method \cite{Bengio04} proposes an out-of-sample solution for unsupervised manifold learning algorithms that embed training samples to coordinates computed from the eigenvectors of a symmetric similarity matrix $M$, such that the entries of the similarity matrix are obtained from a kernel function $K$ as $M_{ij}= K(x_i, x_j)$. Let $Y=[y_1 \ y_2 \ \dots \ y_N]^T$ be the matrix consisting of the embeddings $y_i \in \R^d$ of the training samples $x_i \in \R^n$, such that the $k$-th column $Y_k$ of $Y$ is the $k$-th eigenvector of $M$ with $M \, Y_k = \lambda_k Y_k$. Then, the Nystr\"om method maps a previously unseen test sample $x$ to the point $y(x)=[y^1(x) \ \dots \ y^d(x)]^T$, such that its $k$-th coordinate is given by 
\[
y^k(x) = \frac{1}{\lambda_k} \sum_{i=1}^N Y_{ik} K(x, x_i),
\]
which is shown to extend the embedding of the training samples to the whole ambient space. The Nystr\"om extension can be applied to many common unsupervised manifold learning algorithms including ISOMAP \cite{Tenenbaum00}, LLE \cite{Roweis00}, and Laplacian eigenmaps \cite{Belkin03}, by suitably identifying a kernel-induced similarity matrix $M$ associated with these methods. However, the Nystr\"om method is often inappropriate for the out-of-sample extension of supervised manifold learning methods, since in this case the similarity matrix $M$ is often class-dependent and can no longer be induced from a unique kernel function as $M_{ij}= K(x_i, x_j)$.

Another possible way to obtain the out-of-sample generalization of an embedding is to employ  linear representations. The out-of-sample extension for a test sample $x$ is obtained in \cite{DornaikaR13}, by first computing a sparse representation of $x$ in terms of the training samples as 
\[
x \approx \sum_{i=1}^N a_i x_i 
\]
where $a=[a_1 \ \dots \ a_N]^T$ is a sparse coefficient vector. Regarding the magnitudes of the sparse coefficients as a measure of similarity between $x$ and the training samples, the embedding $y$ of the test sample $x$ is then obtained as a linear combination of the embeddings $y_i$ of the training samples as
\[
y= \frac{\sum_{i=1}^N |a_i | y_i }{\sum_{i=1}^N |a_i |}.
\]

Besides such unsupervised out-of-sample extension methods, the method in \cite{VuralG16} proposes a solution for the out-of-sample problem in a semi-supervised setting. Given a data set containing labeled and unlabeled samples, and the embeddings of the labeled samples learnt via any supervised manifold learning algorithm, the method in \cite{VuralG16} first computes an RBF out-of-sample interpolator that fits the learnt embedding to the labeled training data. This RBF interpolator is then gradually refined using the unlabeled training samples, such that the RBF interpolator and the estimated class labels of the unlabeled samples are jointly updated in an iterative learning procedure.

The above out-of-sample extension strategies can be coupled with several supervised and unsupervised nonlinear manifold learning algorithms, and can be used to extend priorly learnt embeddings to the whole ambient space. Meanwhile, there also exist nonlinear dimensionality reduction algorithms that learn a specific embedding along with its interpolation function extending the embedding to the whole space. The unsupervised manifold learning method \cite{QiaoZWZ13} maps the training samples to a lower-dimensional space with a locally linear reconstruction objective as in the LLE algorithm \cite{Roweis00}; however, under the constraint that the embedding coordinates be polynomial functions of data samples. The polynomial coefficients are thus optimized to minimize the reconstruction error of the locally linear representation. Previously unseen test data can then be embedded into the new domain via the learnt polynomials. Finally, the method in \cite{OrsenigoV12} can be seen as a supervised extension of ISOMAP \cite{Tenenbaum00} that also addresses the problem of extension to novel samples. The training samples are first embedded via a modified version of the ISOMAP algorithm by using a supervised distance function that takes the class labels into account. The learnt embedding is then generalized to the whole space via kernel ridge regression.

The focus of our work is essentially different from that of out-of-sample extension algorithms such as \cite{Bengio04}, \cite{VuralG16}, and \cite{DornaikaR13}, as these methods seek an extension of an already computed embedding to the whole space, while we also address the question of what the embedding should be. Then, compared to manifold learning algorithms such as \cite{QiaoZWZ13} and \cite{OrsenigoV12} that learn an embedding along with its extension, the main difference of our method is that it explicitly takes into account the  performance of the generalization of the learnt classifier to test data, by formulating a supervised learning objective motivated by the recent theoretical generalization bounds of supervised manifold learners.


A possible solution to get around the limitations of linear embeddings while avoiding the out-of-sample problem of nonlinear embeddings is to employ kernel extensions of linear dimensionality reduction methods. The kernel extensions of many well-known dimensionality reduction methods such as PCA, LDA, ICA exist \cite{ScholkopfSM97}, \cite{BaudatA00}, \cite{BachJ02}. The construction of continuous functions via smooth kernels is also quite common in Reproducing Kernel Hilbert Space (RKHS) methods \cite{BelkinNS06}, \cite{Aronszajn50}; however, these methods differ from supervised manifold learning methods  in that the learnt mapping often represents class labels of data samples rather than their coordinates in a lower-dimensional domain of embedding as in manifold learning. The choice of the kernel type and parameters can be critical in kernel methods. Several previous works in the semi-supervised learning literature have addressed the learning of kernels by combining known kernels  \cite{DaiY07}, \cite{ArgyriouHP05}. A two-stage multiple kernel learning method is recently proposed in \cite{TSMKL} for supervised dimensionality reduction, which finds a nonlinear mapping by optimizing between-class and within-class distances.

\section{Theoretical Bounds in Supervised Manifold Learning}
\label{ssec:theo_persp}

Nonlinear dimensionality reduction methods in the literature that minimize objectives  as in \eqref{eq:supLap_formal} often yield embeddings where training samples from different classes are linearly separable, and the local neighborhoods on the same manifold are preserved as imposed by the term involving the within-class graph Laplacian. On the other hand, most existing methods fail to consider how well these embeddings generalize to new test data: When a test sample of unknown class label is mapped to the low-dimensional domain of embedding via an interpolator or an out-of-sample extension method, what is critical is how likely the test sample is to be correctly classified. This depends both on the coordinates of the embedding for the training samples and the interpolator used to generalize the embedding to the whole ambient space. In the previous work \cite{VuralGTR}, this problem is theoretically studied. In this section,  we overview some main results from \cite{VuralGTR}, which will provide a basis for the manifold learning algorithm we propose in Section \ref{sec:proposed_method}.

The classification problem is analyzed in \cite{VuralGTR} in a setting where each data sample in the training set $X =\{ x_i \}_{i=1}^{\numsamp}$   is assumed to belong to one of the classes $\{1, 2, \dots, M\}$ and the samples of each class $m$ are distributed according to the probability measure $\pmes_m $. Let $\M_m$ denote the support of the probability measure $\pmes_m $. Denoting as $B_\delta(x)$ an open ball of radius $\delta$ around a point $x$
\begin{equation*}
B_\delta(x)=\{ u \in \R^n: \| x-u \| < \delta  \},
\end{equation*}
the following definition introduces the smallest possible measure for a ball $B_\delta(x)$ of radius $\delta$ centered around a point in the support $\M_m$ of the $m$-th class.
\begin{equation*}
\minm := \inf_{x \in \M_m} \pmes_m(B_\delta(x))
\end{equation*}

Next, we recall the definition of Lipschitz continuity for a function $f$.

\begin{definition}
A function $f: \R^n \rightarrow \R^d$ is Lipschitz continuous with constant $L>0$ if for any $u, v \in \R^n$
\begin{equation*}
\| f(u) - f(v) \| \leqslant L \, \| u - v \|.
\end{equation*}
\end{definition}


The analysis in \cite{VuralGTR}  considers supervised manifold learning algorithms that compute the embedding $y_i \in \R^d$ of each training sample $x_i \in \R^n$. It is assumed that a test sample $x$ of unknown class label is mapped to $\R^d$ via an interpolation function $f: \R^n \rightarrow \R^d$. The following main result from \cite{VuralGTR} gives a bound on the classification error, when the estimate $\hat \class(x)$ of the class label $\class(x)$ of $x$ is estimated via nearest-neighbor classification in $\R^d$ as $\hat \class(x) = \class (x_i)$, where
\begin{equation*}
i  = \arg \min_j \| y_j - f(x) \|.
\end{equation*}

\begin{theorem}
\label{thm:error_genf_nnclass}

Let $X =\{ x_i \}_{i=1}^{\numsamp} \subset \R^n$ be a set of training samples such that each $x_i$ is drawn i.i.d.~from one of the probability measures $\{\pmes_m \}_{m=1}^M$, with $\pmes_m$ denoting the probability measure of the $m$-th class. Let $Y=\{ y_i \}_{i=1}^{\numsamp}$ be an embedding of $X$ in $\R^d$ such that there exist a constant $\mar>0$ and a constant $\Demb_\delta$ depending on $\delta>0$ satisfying
\begin{equation*}
\begin{split}
\| y_i - y_j \| &< \Demb_\delta, \text{ if } \| x_i - x_j  \| \leqslant 2 \delta \text{ and } \class(x_i)=\class(x_j) 
\\
\| y_i - y_j \| &> \mar,  \  \text{ if } \class(x_i) \neq \class(x_j).
\end{split}
\end{equation*}
For given $\epsilon>0$ and $\delta>0$, let $f: \R^n \rightarrow \R^d$ be a Lipschitz-continuous interpolation function with constant $L$, which maps each $x_i$ to $f(x_i) =y_i$, such that 
\begin{equation}
\label{eq:cond_sep_genf_nn}
L \delta + \sqrt{d} \epsilon +   \Demb_{\delta} \leqslant \frac{ \mar}{ 2}.
\end{equation}
Consider a test sample $x$ randomly drawn according to the probability measure $\pmes_m$ of class $m$. For any $\Knb>0$, if $X$ contains at least $\numsamp_m$ training samples from the $m$-th class drawn i.i.d.~from $\pmes_m$  such that 
\begin{equation*}
\numsamp_m > \frac{\Knb}{\minm}
\end{equation*}
then the probability of correctly classifying $x$ with nearest-neighbor classification in  $\R^d$ is lower bounded as
\begin{equation}
\label{eq:prob_lb_rbfint}
P\left( \hat{\class}(x) = m \right) \geqslant 
1 - \exp \left( - \frac{2 \, (\numsamp_m \, \minm - \Knb)^2 }{\numsamp_m} \right)
 - 2 d \exp \left( - \frac{  \Knb \, \epsilon^2 }{ 2 L^2 \delta^2} \right).
\end{equation}
\end{theorem}

Theorem \ref{thm:error_genf_nnclass} considers an embedding such that nearby training samples from the same class are mapped to nearby coordinates, while training samples from different classes are separated by a distance of at least $\mar$ in the low-dimensional domain of embedding. An illustration of the setting considered in the theorem is given in Figure \ref{fig:illus_embedding}. The parameter $\mar$ can be considered as the separation margin of the embedding. Then for such an embedding, the condition in \eqref{eq:cond_sep_genf_nn} assumes an interpolator $f$ that is sufficiently regular (with a sufficiently small Lipschitz constant $L$) compared to the separation margin $\mar$. Finally, a probabilistic classification guarantee is given for this setting in \eqref{eq:prob_lb_rbfint}, which states that the misclassification probability decreases exponentially with the number of samples under these assumptions. The above result considers the NN classifier in the final stage, which is a simple and widely used classification strategy for which efficient algorithms exist \cite{VajdaS16}. An extension of this result is also presented in \cite{VuralGTR} which studies the performance of classification when a linear classifier is used instead of nearest-neighbor classification in the low-dimensional domain. If a linear classifier is used in the domain of embedding, a very similar condition to  \eqref{eq:cond_sep_genf_nn} relating the interpolator regularity to the separation margin is obtained, which yields a similar probabilistic bound on the misclassification error.

\begin{figure}[t]
  \centering
  \centerline{\includegraphics[width=6.0cm]{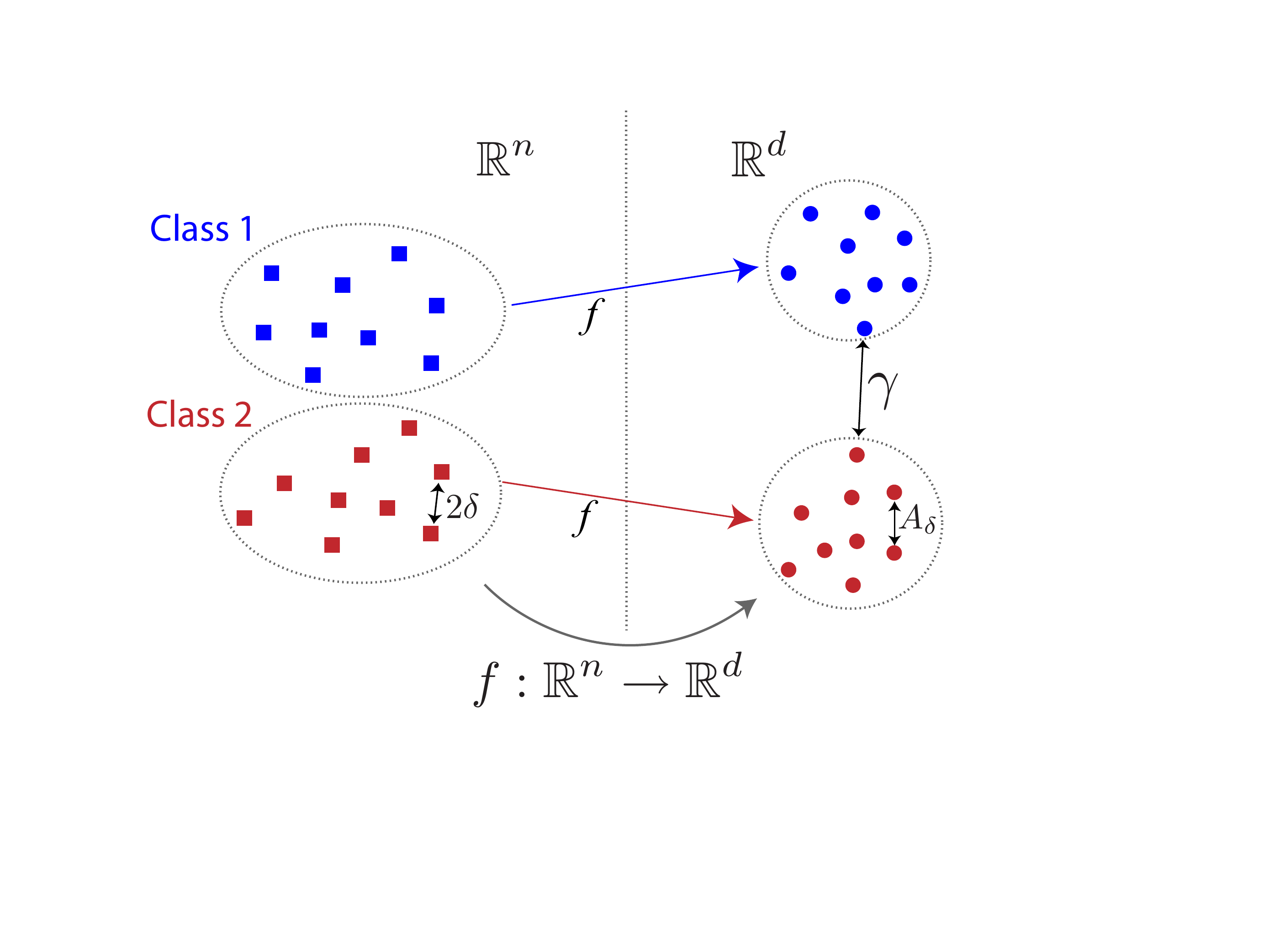}}
  \caption{Illustration of the setting considered in Theorem \ref{thm:error_genf_nnclass}. Samples from the same class having a distance of $2\delta$ are embedded to points at most $\Demb_\delta$ apart. Samples from different classes are separated by at least $\mar$.}\medskip
  \label{fig:illus_embedding}
\end{figure}

While most supervised manifold learning methods in the literature focus on achieving a large separation between the training samples from different classes in the embedding, the condition  \eqref{eq:cond_sep_genf_nn} in the above theoretical analysis points to a critical compromise that must be sought in supervised dimensionality reduction: Achieving high separation between different classes in the training set does not necessarily mean that the classifier will generalize well to test samples. The presence of a sufficiently regular interpolator is furthermore needed, so that the Lipschitz constant $L$ of the interpolator remains below a threshold involving the separation margin $\mar$ of the embedding. From this perspective, depending on the data distribution, increasing the separation too much has the risk of forcing the interpolator to be too irregular, which may in turn cause condition  \eqref{eq:cond_sep_genf_nn} to fail. What we propose in this work is to learn the embedding $\{ y_i \}_{i=1}^N$ together with the interpolator $f$ in view of the condition \eqref{eq:cond_sep_genf_nn}, which is detailed in the next section.

\section{Proposed Nonlinear Supervised Smooth Embedding Method}
\label{sec:proposed_method}

In this section, we present our proposed supervised dimensionality reduction method. We first formulate the manifold learning problem and define an optimization problem based on the perspectives discussed in Section \ref{ssec:theo_persp}. We then describe our algorithm.

\subsection{Formulation of the manifold learning problem}

Given training points $X=\{x_i\}_{i=1}^N\subset \mathbb R^n$ from $M$ classes, our purpose is to learn an embedding of data $Y=\{y_i\}_{i=1}^N\subset \mathbb R^d$ together with a continuous interpolation function $f : \mathbb R^n \rightarrow \mathbb R^d$ , such that $f(x_i)=y_i$. The interpolator $f$ will then be used to classify new  test points $x$ by mapping $x$ to the low-dimensional domain $\R^d$ as $f(x)$, so that examining $f(x) \in \mathbb R^d$ with respect to the embedding $Y$ of the training points with known class labels provides an estimate of the class label of $x$.

Our method relies on the theoretical results presented in Section \ref{ssec:theo_persp}. Recall from Theorem \ref{thm:error_genf_nnclass} that, a necessary condition to obtain good generalization performance is
\[
L \delta + \sqrt{d} \epsilon +   \Demb_{\delta} \leq \frac{ \mar}{ 2}.
\]
In the sequel, we formulate a manifold learning problem in view of this condition, whose purpose is to make the the Lipschitz constant $L$ of the interpolator and the distance $\Demb_{\delta}$ between neighboring points from the same class as small as possible, while making the separation $\gamma$ between different classes as large as possible, in order to increase the chances that the above condition be met.


Let $f(x) = [f^1(x) \ \dots \ f^d(x)] \in \mathbb R^d$, where $f^k(x)$ is the $k$-th dimension of $f(x)$, with $f^k:\mathbb R^n \rightarrow \mathbb R$. We propose to choose the function $f$ as a radial basis function (RBF) interpolator, as RBF interpolators are a well-studied family of functions \cite{Baxter92}, \cite{Piret07} with many desirable properties such as smoothness and adjustable spread around anchor points. Hence, each component $f^k$ of $f$ is of the form
\begin{equation}
\label{eq:rbf_interpolator}
f^k(x) = \sum_{i=1}^N  c_i^k  \phi (\|x-x_i\|)  
\end{equation}
where  $\phi : \R \rightarrow \R^+$ is an RBF kernel, $c_i^k$ are the coefficients, and $x_i$ are the kernel centers. A common choice for the RBF kernel is the Gaussian kernel $\phi (r) = e^{-r^2/\sigma ^2}$, which we also adopt in this work. Under this setting, we now examine our three entities of interest, namely the regularity of the interpolator, the distance between neighboring points from the same class and the separation between different classes.\\

\textit{Interpolator regularity.}
We begin with proposing a Lipschitz constant for $f$ in terms of the function parameters.

\begin{proposition}
Let $L_\phi := \sqrt{2} e^{-\half} \sigma^{-1}$ and let $C$ be the matrix consisting of the RBF coefficients such that $C_{ij}=c_i^j$. Then the RBF interpolator $f$ satisfies for all $u, v \in \R^n$
\[
\|f(u) - f(v)\| \leqslant L\|u-v\|
\]
where $L:=\sqrt{N} L_\phi {\|C\|}_F$.
\end{proposition}

\begin{proof}
We first show that $L_\phi$ is a Lipschitz constant of the RBF kernel.
%
%
It is easy to show that the derivative of $\phi (r) = e^{-r^2/\sigma ^2}$ takes is maximum magnitude at $r=\sigma/\sqrt{2}$, so that for all $r$
\[
\left | \frac{d \phi (r)}{ d r} \right | \leqslant \left | \frac{d \phi (t)}{ d t} \bigg|_{t=\sigma/\sqrt{2}}  \right | 
= \sqrt{2} e^{-\half} \sigma^{-1} = L_\phi.
\]
This upper bound on the derivative magnitude shows that $L_\phi = \sqrt{2} e^{-\half} \sigma^{-1}$ is a Lipschitz constant for the Gaussian kernel as
\[  |\phi(u) - \phi(v)| \leqslant L_\phi |u-v| \]
for all $u, v \in \mathbb R$.

 Then we derive the Lipschitz constant of $f(x)$ as follows. First, observe that
\begin{equation*}
\begin{split}
|f^k(u) - f^k(v)| &= \left| \sum_{i=1}^N c_i^k \big(  \phi (\|u-x_i\|) - \phi(\|v-x_i\|) \big) \right|  \\
&\leqslant \sum_{i=1}^N |c_i^k| \, \left| \phi(\|u-x_i\|) - \phi(\|v-x_i\|) \right|
\end{split}
\end{equation*}
where the second term inside the sum can be upper bounded as
\[
|\phi(\|u-x_i\|) - \phi(\|v-x_i\|)| \leqslant L_\phi \big| \|u-x_i\| - \|v-x_i\|  \big | \leqslant L_\phi \|u-v\|.
\]
This gives in the above equation
\[
|f^k(u) - f^k(v)| \leqslant L_\phi \sum_{i=1}^N |c_i^k| \|u-v\| 
\leqslant  L_\phi  \| c^k \|_1 \| u-v \|
\leqslant \sqrt{N} L_\phi \, \|c^k\| \|u-v\|
\]
where $ c^k = [ c_1^k \dots c_N^k ]^T$ denotes the coefficient vector of the function $f^k$ and $\| \cdot \|_1$ is the $\ell_1$-norm of a vector. Then we have 
\begin{equation*}
\begin{split}
\|f(u) - f(v)\| &= \sqrt{\sum_{k=1}^d{|f^k(u) - f^k(v)|}^2} \leqslant \sqrt{\sum_{k=1}^d N L_\phi^2 {\|c^k\|}^2 {\|u-v\|}^2}  \\
&= \sqrt{N} L_\phi \|u-v\| \sqrt{\sum_{k=1}^d{\|c^k\|}^2} 
=   \sqrt{N} L_\phi \|C\|_F   \|u-v\|
\end{split}
\end{equation*}
%
%
where $\| \cdot \|_F$ denotes the Frobenius norm of a matrix. Hence, we get
\[
\|f(u) - f(v)\| \leqslant L\|u-v\|
\]
where $L=\sqrt{N} L_\phi {\|C\|}_F$ is the Lipschitz constant of $f(x)$.
%
%
%
\end{proof}

When learning an interpolator, we would like to minimize the Lipschitz constant $L=\sqrt{N} L_\phi {\|C\|}_F$ of $f(x)$. From the form \eqref{eq:rbf_interpolator} of the interpolator components and the fact that the interpolator values at training points must correspond to the coordinates of the embedding $y_i = f(x_i) $, we get the relation
\[
\Psi C= Y
\] 
where $\Psi \in \R^{N \times N}$ is the matrix consisting of the values of the RBF kernels with $\Psi_{ij}=\phi(\|x_i-x_j\|)$
%
%
and $Y = [ y_1 \ y_2 \ \dots \ y_N ]^T \in \R^{N \times d}$ is the matrix consisting of the coordinates of the embeddings of the training samples. Then the coefficient matrix is given by $ C = \Psi^{-1} Y$, so that
\begin{equation}
\label{eq:Cmat_normsq}
\|C\|_F^2 = {\|\Psi^{-1} Y\|}_F^2 = \tr(Y^T \Psi^{-2} Y).
\end{equation}

In order to keep the Lipschitz constant $L=\sqrt{N} L_\phi {\|C\|}_F$ of the interpolator small in the learnt embedding, we need to keep both the Lipschitz constant $L_\phi$ of the Gaussian kernel and the norm $\|C\|_F$ of the coefficient matrix small. Using the expression of $\|C\|_F^2$ in \eqref{eq:Cmat_normsq} and recalling that $L_\phi =\sqrt{2} e^{-\half} \sigma^{-1} $, we thus propose to minimize the following objective for controlling the interpolator regularity
\begin{equation}
\label{eq:lipschitz_obj}
\min_{Y, \sigma} \ \tr(Y^T \Psi^{-2} Y) +  \frac{\mu}{\sigma^2}
\end{equation}
where $\mu$ is a weight parameter. The objective is chosen proportionally to the squares of  the terms $\| C \|_F$ and $L_\phi$ instead of themselves, due to the convenience of the analytical expression obtained for $\| C \|_F^2$  in  \eqref{eq:Cmat_normsq}.\\

\textit{Distance between neighboring points from the same class.}
Recall from Theorem \ref{thm:error_genf_nnclass} that the condition \eqref{eq:cond_sep_genf_nn} required for good classification performance enforces the term $ \Demb_{\delta}$ to be sufficiently small, where $\Demb_\delta$ is an upper bound on the distance between the embeddings of nearby samples; i.e., $\| y_i - y_j \| < \Demb_\delta$ whenever  $\| x_i - x_j  \| \leqslant 2 \delta$. It is not easy to study the distance $\| y_i - y_j \|$ in relation with the ambient space distance $ \|x_i - x_j \| $ for each pair of samples $x_i$, $x_j$. Nevertheless, we adopt a constructive solution here and relax this problem to the minimization of the distance between the embeddings of nearby points from the same class. The total distance between the embeddings of neighboring points from the same class, weighted by the edge weights, is given by 
\[
 \sum_{x_i, x_j: \ C(x_i) = C(x_j)} {\|y_i - y_j\|}^2 w_{ij} = \tr(Y^TL_wY).
\]
Here $L_w=D_w - W_w$ is the within-class Laplacian matrix associated with the within-class weight matrix $W_w$, where $D_w$ is the diagonal degree matrix whose entries are given by $(D_w)_{ii}=\sum_{j} (W_w)_{ij}$. The within-class weight matrix $W_w$ contains the weights $w_{ij}$ of the edges between each pair of neighboring samples $\quad x_i \sim x_j$ from the same class. A common choice for assigning the edge weights is the Gaussian kernel, in which case the matrix $W_w$ is of the form
\begin{align*}
(W_w)_{ij}=\begin{cases}
e^{-\frac{\| x_i - x_j \|^2}{\beta}}, &\text{ if } C(x_i) = C(x_j), \quad x_i \sim x_j\\
0, &\text{ otherwise.}
\end{cases} 
\end{align*}
Then, the objective
\begin{equation}
\label{eq:D_2delta_obj}
\min_Y \ \tr(Y^TL_wY)
\end{equation}
used in several previous works as discussed in Section \ref{sec:related_work} is an appropriate choice for our purpose.\\
 
\textit{Separation between samples from different classes.} The last entity to be examined in view of the condition \eqref{eq:cond_sep_genf_nn} is the separation margin $\mar$. In order to satisfy the condition \eqref{eq:cond_sep_genf_nn}, the separation between the samples from different classes must be sufficiently high. Although the margin $\mar$ stands for a lower bound for the distance $\| y_i - y_j \|$ between any pair of samples from different classes in Theorem \ref{thm:error_genf_nnclass}, the examination of the minimum value of $\| y_i - y_j \|$ for all pairs of samples is a relatively hard problem. We propose to relax this and evaluate the total distance between the embeddings of different-class samples in this study. Hence, in order to increase the separation margin $\mar$, we propose to maximize
\[
\sum_{C(x_i) \neq C(x_j)} {\|y_i - y_j\|}^2 = \tr(Y^TL_bY)
\]
where $W_b$ is a between-class weight matrix given by
\begin{align*}
(W_b)_{ij}=\begin{cases}
1 &\text{ if } C(x_i) \neq C(x_j) \\
0 &\text{ if } C(x_i) = C(x_j), 
\end{cases} 
\end{align*}
the diagonal between-class degree matrix is defined as $(D_b)_{ii}=\sum_{j} (W_b)_{ij}$, and $L_b = D_b - W_b$ is the corresponding between-class Laplacian matrix. Thus, the maximization of the separation margin is represented by the objective function
\begin{equation}
\label{eq:gamma_obj}
\max_Y \tr(Y^TL_bY).
\end{equation}\\
%

%
%
%
%
%
%
%
%
%
%
%
%
%
%
%
%

\textit{Overall optimization problem.} Now, bringing together the objective functions presented in \eqref{eq:lipschitz_obj}, \eqref{eq:D_2delta_obj}, and \eqref{eq:gamma_obj}, we propose to solve the following optimization problem in order to learn an embedding $Y$ together with its corresponding interpolator:
\begin{equation}
\label{eq:overall_obj}
\min_{Y, \sigma} 
\tr(Y^T L_w Y) - \mu_1 \tr(Y^T L_b Y) + \mu_2  \tr(Y^T \Psi^{-2} Y) + \frac{\mu_3}{\sigma^2}, 
\text{ s.t. } Y^T Y = I
\end{equation}
Here $\mu_1$, $\mu_2$, and $\mu_3$ are positive weights that balance the different terms in the objective function, and the normalization condition $Y^T Y = I$ is imposed in order to prevent solutions with arbitrarily small embedding coordinates that might trivially minimize the objective.

\subsection{Proposed manifold learning algorithm}
\label{ssec:proposed_method}

The proposed objective function \eqref{eq:overall_obj} can be made convex with respect to $Y$ if the weight parameters $\mu_1$ and $\mu_2$ are suitably chosen; however, it is not jointly convex with respect to both optimization variables $Y$ and $\sigma$. We thus  propose to minimize  \eqref{eq:overall_obj} with an alternating iterative optimization algorithm. In each iteration, we first fix the scale parameter $\sigma$ and optimize the embedding coordinates $Y$, which is then followed by fixing $Y$ and optimizing $\sigma$. 

\textit{Optimization of $Y$.} When the scale parameter $\sigma$ is fixed, the minimization of the objective \eqref{eq:overall_obj} is equivalent to the following optimization problem
\begin{equation}
\label{eq:opt_Y}
\begin{split}
Y^* &= \arg \min_Y \ \ \tr(Y^T L_w Y) - \mu_1 \tr(Y^T L_b Y) + \mu_2  \tr(Y^T \Psi^{-2} Y) \text{ s.t. }
Y^T Y = I \\
&= \arg \min_Y \ \  \tr \left(  Y^T (L_w - \mu_1 L_b + \mu_2 \Psi^{-2}  ) Y \right) \text{ s.t. }
Y^T Y = I.\\
\end{split}
\end{equation}
The solution to this problem is given by the $N \times d$ matrix $Y^*$ whose $k$-th column consists of the eigenvector of the matrix 
\begin{equation}
\label{eq:defn_A_mat}
A=L_w - \mu_1 L_b + \mu_2 \Psi^{-2} 
\end{equation}
that corresponds to its $k$-th smallest eigenvalue, for $k=1, \dots, d$. 

\textit{Optimization of $\sigma$.} Note that the dependence of the objective function  \eqref{eq:overall_obj} on the scale parameter $\sigma$ is through its third term $\mu_2  \tr(Y^T \Psi^{-2} Y)$ and fourth term $\mu_3/{\sigma^2}$. Hence, when the embedding $Y$ is fixed, the optimization of the objective is reduced to the problem
\begin{equation}
\label{eq:opt_sigma}
\sigma^* = \arg \min_{\sigma}  \  \ 
\mu_2 \tr(Y^T \Psi^{-2} Y) + \frac{\mu_3}{\sigma^2}.
\end{equation}
The objective in \eqref{eq:opt_sigma} is not a convex function of $\sigma$ in general. Nevertheless, a useful observation is the following: As the entries of the matrix $\Psi$ consist of the RBF kernel terms
\[
\phi(\|x_i-x_j\|)  = \exp \left(-  \frac{\|x_i-x_j\|^2}{\sigma ^2} \right)
\]
the matrix $\Psi$ and its inverse $\Psi^{-1}$ have poor conditioning when $\sigma$ takes arbitrarily large values. Hence, the first term $\tr(Y^T \Psi^{-2} Y)$ in \eqref{eq:opt_sigma} increases with increasing large values of $\sigma$. On the other hand, the term $\sigma^{-2}$ approaches infinity as $\sigma$ approaches 0. These observations imply that that there exists a positive kernel scale $\sigma^*>0$ that minimizes the objective \eqref{eq:opt_sigma}. As the problem \eqref{eq:opt_sigma} requires the optimization of a single parameter $\sigma$, an optimal value $\sigma^*$ can be computed easily. In practice, we find $\sigma^*$ via an exhaustive search procedure, by computing the objective \eqref{eq:opt_sigma} over a sufficiently dense sampling of the $\sigma$ values within a suitably chosen interval $[\sigma_{\min}, \sigma_{\max}]$ of typical scale parameters. The optimal parameter $\sigma^*$ is then taken as the $\sigma$ value at which the objective is minimum.

\begin{algorithm}[t]
\caption{Nonlinear Supervised Smooth Embedding (NSSE)}
\begin{algorithmic}[1]
\STATE
\textbf{Input:} \\
$X=\{x_i\}_{i=1}^N \subset \R^n$: Training samples with known class labels\\
$d$: Embedding dimension\\
$\mu_1, \mu_2, \mu_3$: Weight parameters

\STATE
\label{alg1:init}
\textbf{Initialization:} Set kernel scale $\sigma^*$ to a typical positive value.  

\REPEAT
\label{alg1:main_loop_begin}

\STATE 
\label{alg:opt_Y}
Fix $\sigma=\sigma^*$ and optimize $Y$ by solving

$Y^* = \arg \min_Y \tr \left(Y^T  ( L_w-\mu_1 L_b + \mu_2 \Psi^{-2} ) Y \right) \text{ s.t. } Y^T Y = I $
%

\STATE
\label{alg:opt_sigma}
Fix $Y = Y^*$ and optimize $\sigma$ by solving

$\sigma^* = \arg \min_{\sigma} \ \mu_2 \tr(Y^T \Psi^{-2} Y)  +  \mu_3 \sigma^{-2}  $

\UNTIL{Objective function in \eqref{eq:overall_obj} is stabilized}
\label{alg1:main_loop_end}


\STATE
\label{alg1:finalization}
Compute interpolator coefficients as $C= \Psi^{-1} Y$.


\STATE
\textbf{Output}:\\
$Y=\{y_i\}_{i=1}^N\subset \R^d$: Embedding of training samples\\
$f: \R^n \rightarrow \R^d$: Interpolation function  
\end{algorithmic}
\label{alg:Optimization_of_Y_and_sigma}
\end{algorithm}

These steps for the alternating optimization of $Y$ and $\sigma$ are applied successively until the stabilization of the objective function. Note that if $\mu_1$ is chosen sufficiently small to make the matrix $A$ in \eqref{eq:defn_A_mat} positive semi definite, the overall objective function \eqref{eq:overall_obj} is positive. In this case, since both of the alternating optimization steps in \eqref{eq:opt_Y} and \eqref{eq:opt_sigma} bring updates that cannot  increase the objective function in each iteration, being bounded from below, the objective function is guaranteed to converge.

Once the embedding $Y$ of the training points and the kernel scale parameter $\sigma$ are computed in this way, the interpolation function $f$ is simply obtained as in \eqref{eq:rbf_interpolator} by computing the coefficients $c_i^k$ as $C=\Psi^{-1} Y$. We call the proposed method Nonlinear Supervised Smooth Embedding (NSSE) and give its description in Algorithm \ref{alg:Optimization_of_Y_and_sigma}.

\subsection{Complexity of the proposed algorithm}
\label{ssec:comp_analysis}

We now analyze the computational complexity of the proposed NSSE method.  The algorithm is composed of three main stages, which are the initialization stage (calculation of the $L_w$ and $L_b$ matrices), the main loop between steps \ref{alg1:main_loop_begin} and \ref{alg1:main_loop_end} of Algorithm \ref{alg:Optimization_of_Y_and_sigma}, and the finalization stage in step \ref{alg1:finalization}. 

 
In the initialization step, the complexity of the computation of $L_w$ and $L_b$ is mainly determined by the complexity of computing the within-class and between-class weight matrices $W_w$ and $W_b$, which is of $O(n N^2)$. 

We next consider the main loop of the algorithm. The matrix $\Psi$ in step \ref{alg:opt_Y} can be calculated with complexity $O(n N^2)$ and it is inverted with complexity $O(N^3)$ to obtain $\Psi^{-1}$. As a result, the computation of $\Psi^{-2}$ is of complexity $O(n N^2) + O(N^3)$. In order to find $Y^*$ in step \ref{alg:opt_Y}, the eigenvectors of $L_w - \mu_1 L_b + \mu_2 \Psi^{-2}$ should be found, which is of complexity $O(N^3)$. Consequently, the total complexity of step \ref{alg:opt_Y} is $O(n N^2) + O(N^3)$. In step \ref{alg:opt_sigma}, the expression $\mu_2 \tr(Y^T \Psi^{-2} Y)  +  \mu_3 \sigma^{-2}$ must be computed repeatedly to find $\sigma^*$, which is of complexity $O(N^3)$. Hence, the complexity of the main loop of the algorithm is found as  $O(n N^2) + O(N^3)$.

In step \ref{alg1:finalization}, the complexity of the calculation of $\Psi^{-1}$ is $O(N^3)$, and the matrix product $\Psi^{-1} Y$ is of complexity $O(d N^2)$. We may assume $d \ll N$, which then gives the complexity of step \ref{alg1:finalization} as of $O(N^3)$. Combining this with the previous stages, the overall complexity of the algorithm is found as $O(n N^2) + O(N^3)$.




\section{Experimental Results}
\label{sec:exp_results}

In this section, we evaluate the performance of the proposed NSSE method on six real data sets. We first describe the data sets, then study the iterative optimization procedure employed in the proposed method, and then compare the performance of NSSE with that of other supervised manifold learning algorithms and traditional classifiers. 

\subsection{Data sets and experimentation setting}
\label{ssec:dataset}

We experiment on the data sets listed below. Some sample images from one class of each data set are presented in Figure \ref{fig:smpl_images}.

\textit{Yale Face Database.} The data set consists of 2242 greyscale face images of 38 different subjects, where each subject has 59 images \cite{GeBeKr01}. All images are taken from a single viewpoint with variations in the lighting angles and lighting rates. 




\textit{COIL-20 Database.} The Columbia Object Image Library database consists of 1440 grayscale images of 20 different objects, where each object has 72 images captured by rotation increments of 5 degrees \cite{NeneNM96}. 




\textit{ORL Database.} The database consists of a total of 400 images, with 10 images of each one of the 40 subjects taken in an upright, frontal position \cite{SamariaH94}. The images contain variations in the the lighting, facial expressions and facial details such as glasses. 



\textit{FEI Database.} The FEI database is a Brazilian face database containing a total of 2800 images, with 14 images for each one of the 200 subjects taken in an upright frontal position with profile rotation of up to about 180 degrees and scale variation of about 10\% \cite{ThomazG10}. We experiment on 50 classes from this database. 


 
\textit{ROBOTICS-CSIE Database.} The database contains a total of 3330 grayscale face images of 90 subjects, with 37 images for each subject captured under rotation increments of 5 degrees \cite{CSIEdata}. We experiment on 40 classes from this database. 




\textit{MIT-CBCL Database.} The database contains face images of 10 subjects \cite{MITCBCL}. We experiment on a total of 5240 images, with 524 images per subject captured under rotations of up to 30 degrees and varying illumination conditions. 



\begin{figure}[t]
\begin{center}
     \subfigure[Yale database]
       {\label{fig:yale_smpls}\includegraphics[height=0.8cm]{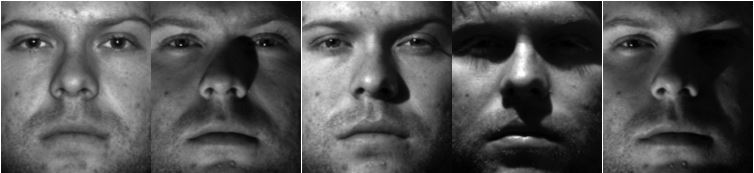}}
     \subfigure[COIL-20 database]
       {\label{fig:coil20_smpls}\includegraphics[height=0.8cm]{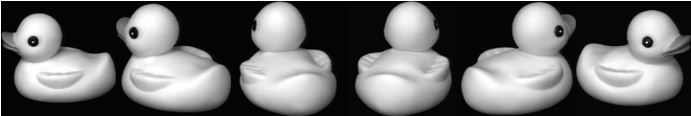}}
      \subfigure[ORL database]
       {\label{fig:orl_smpls}\includegraphics[height=0.8cm]{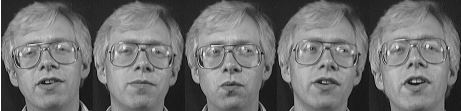}}
     \subfigure[FEI database]
       {\label{fig:fei_smpls}\includegraphics[height=0.8cm]{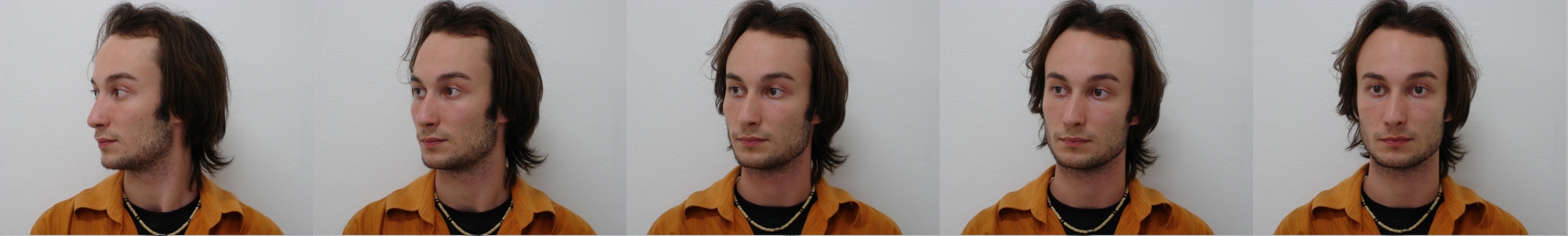}}
     \subfigure[ROBOTICS-CSIE database]
       {\label{fig:csie_smpls}\includegraphics[height=0.8cm]{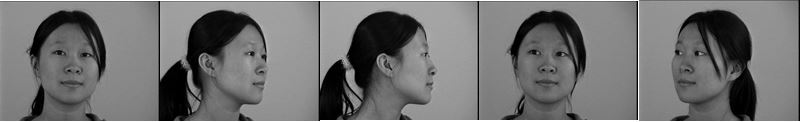}}
     \subfigure[MIT-CBCL database]
       {\label{fig:mitcbcl_smpls}\includegraphics[height=0.8cm]{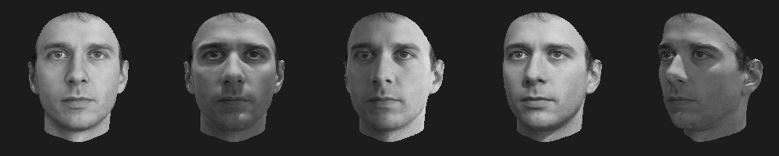}}
 
       \end{center}
 \caption{Sample images from one class of the used databases}
 \label{fig:smpl_images}
\end{figure}

We experiment on greyscale versions of the images resized to around $25 \times 25$ pixels. All experiments are conducted in a supervised setup, by randomly separating the images into a training set and a test set in each repetition of the experiment. In all experiments, the proposed NSSE algorithm is evaluated in a setting where the training images are used to learn a continuous embedding into a low-dimensional domain. The test images are then  mapped to the domain of embedding via the learnt interpolator and their class labels are estimated via nearest neighbor classification in the low-dimensional domain. The graph edge weights are set with a Gaussian kernel. In all experiments, the weight parameters $\mu_1$, $\mu_2$, and $\mu_3$ of NSSE are set with cross-validation. The weight parameters are set sequentially, by first initializing them with some typical values and then optimizing one of them at a time via cross validation where the others are kept fixed. When optimizing one weight parameter, the training samples are divided randomly into two sets as the training set and the validation set, the algorithm is trained on the training set, and the classification error is measured on the validation set for different values of the weight parameter. We repeat this several times by randomly assigning the training and the validation set, and then finally select the parameter value that gives the smallest average classification error on the validation set.  In practice, we have observed that the typical ranges of appropriate $\mu_1$, $\mu_2$, and $\mu_3$ values do not usually vary dramatically between different data sets and setting these parameters to values within the intervals $\mu_1 \in [100,1000]$, $\mu_2 \in [0.0001, 0.001]$, and $\mu_3 \in [1,5]$ often yields satisfactory performance.


\subsection{Study of the iterative optimization procedure}

\begin{figure}[t]
\begin{center}
     \subfigure[]
       {\label{fig:obj_vs_iters_fei_high}\includegraphics[scale=0.29]{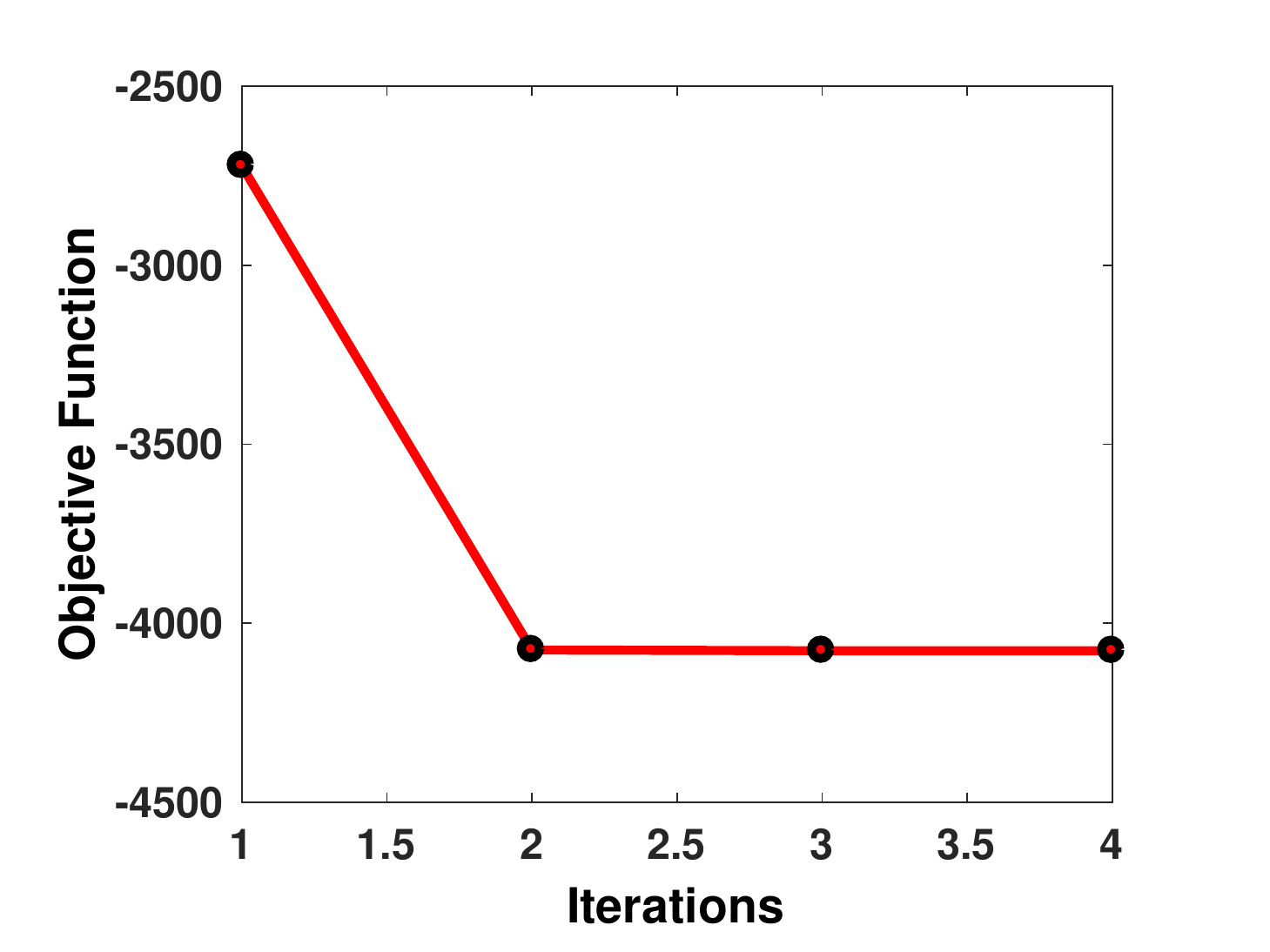}}
       \hspace{-0.6cm}
     \subfigure[]
       {\label{fig:msc_vs_iters_fei_high}\includegraphics[scale=0.29]{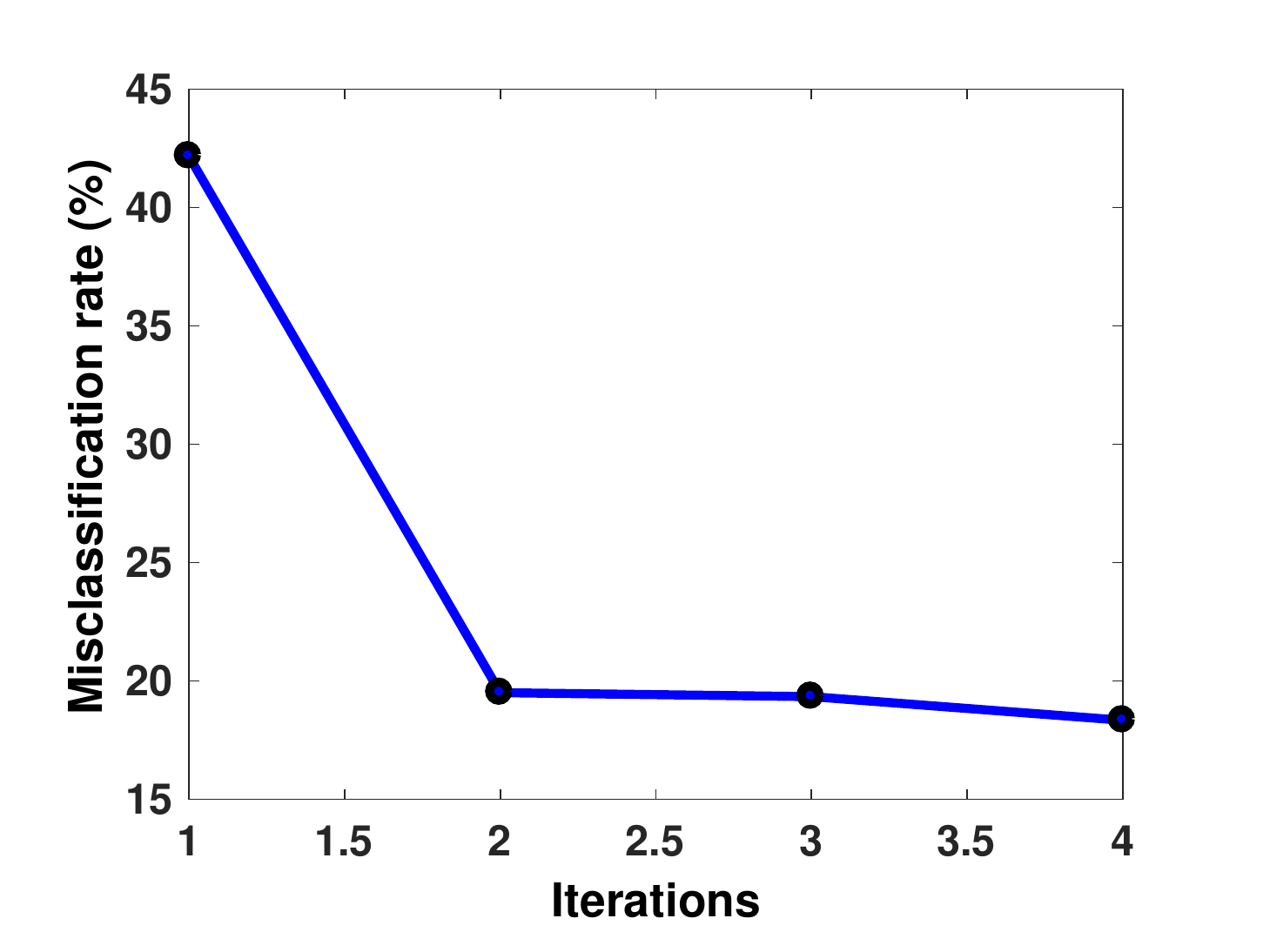}}
       \hspace{-0.6cm}
     \subfigure[]
       {\label{fig:sgm_vs_iters_fei_high}\includegraphics[scale=0.29]{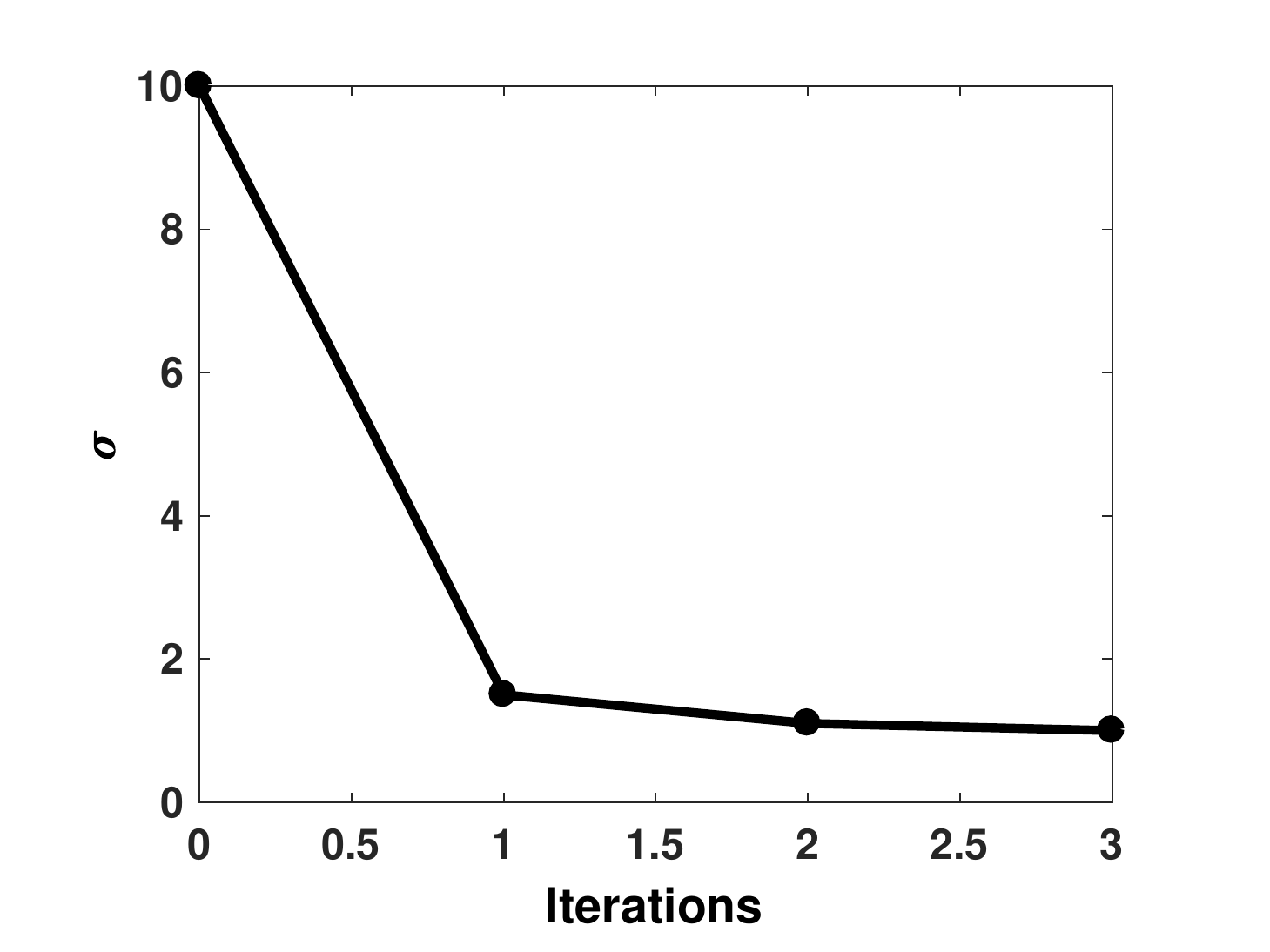}}
       \end{center}
 \caption{Algorithm performance throughout iteratations when initialized with a high RBF kernel scale. (a) Convergence of the objective function (b) Stabilization of the misclassification rate (c) Stabilization of the RBF kernel scale $\sigma$ }
 \label{fig:exp_alg_conv_fei_high}
\end{figure}

\begin{figure}[!h]
\begin{center}
     \subfigure[]
       {\label{fig:obj_vs_iters_fei_low}\includegraphics[scale=0.29]{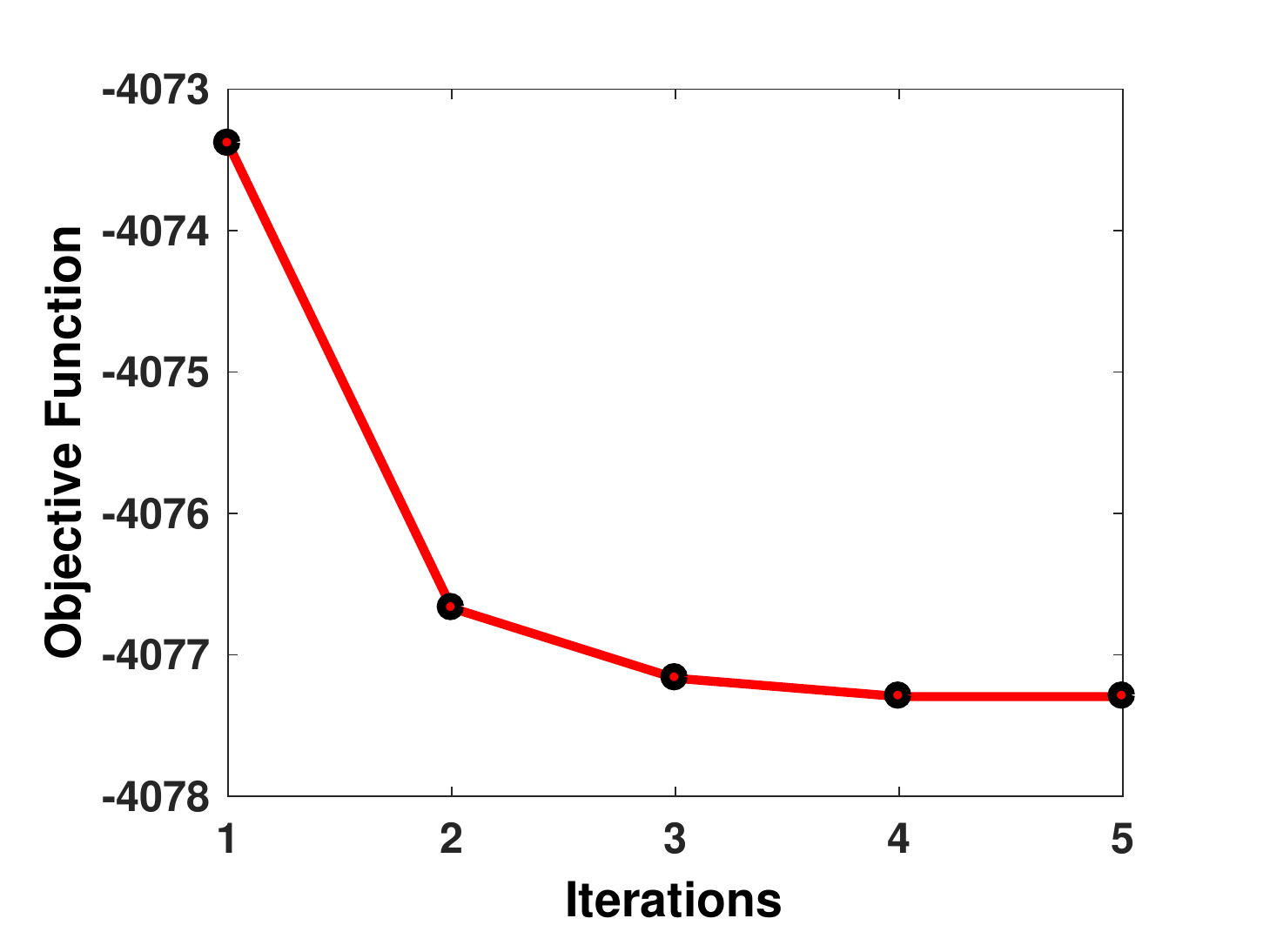}}
      \hspace{-0.6cm}
     \subfigure[]
       {\label{fig:msc_vs_iters_fei_low}\includegraphics[scale=0.29]{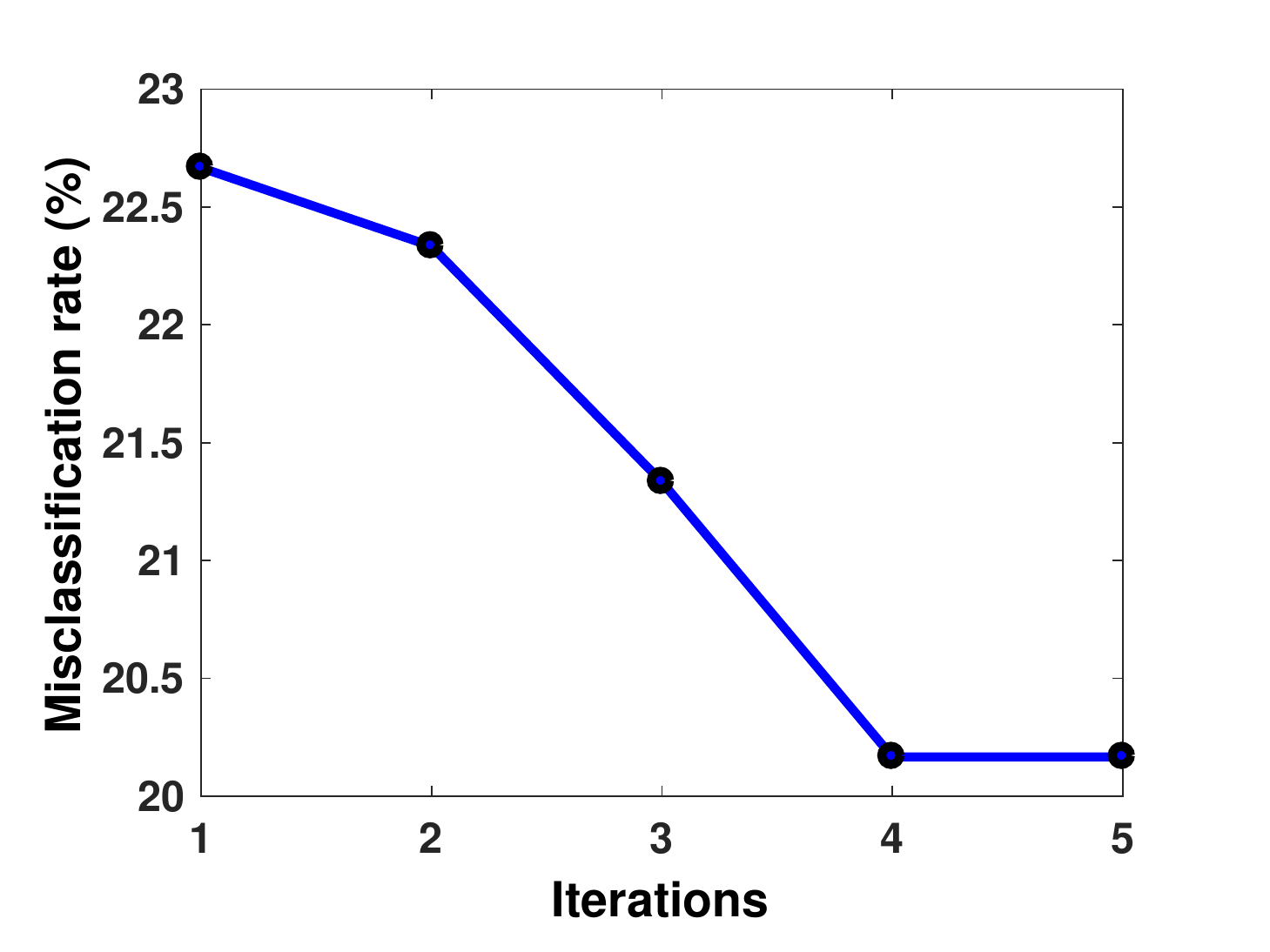}}
      \hspace{-0.6cm}
     \subfigure[]
       {\label{fig:sgm_vs_iters_fei_low}\includegraphics[scale=0.29]{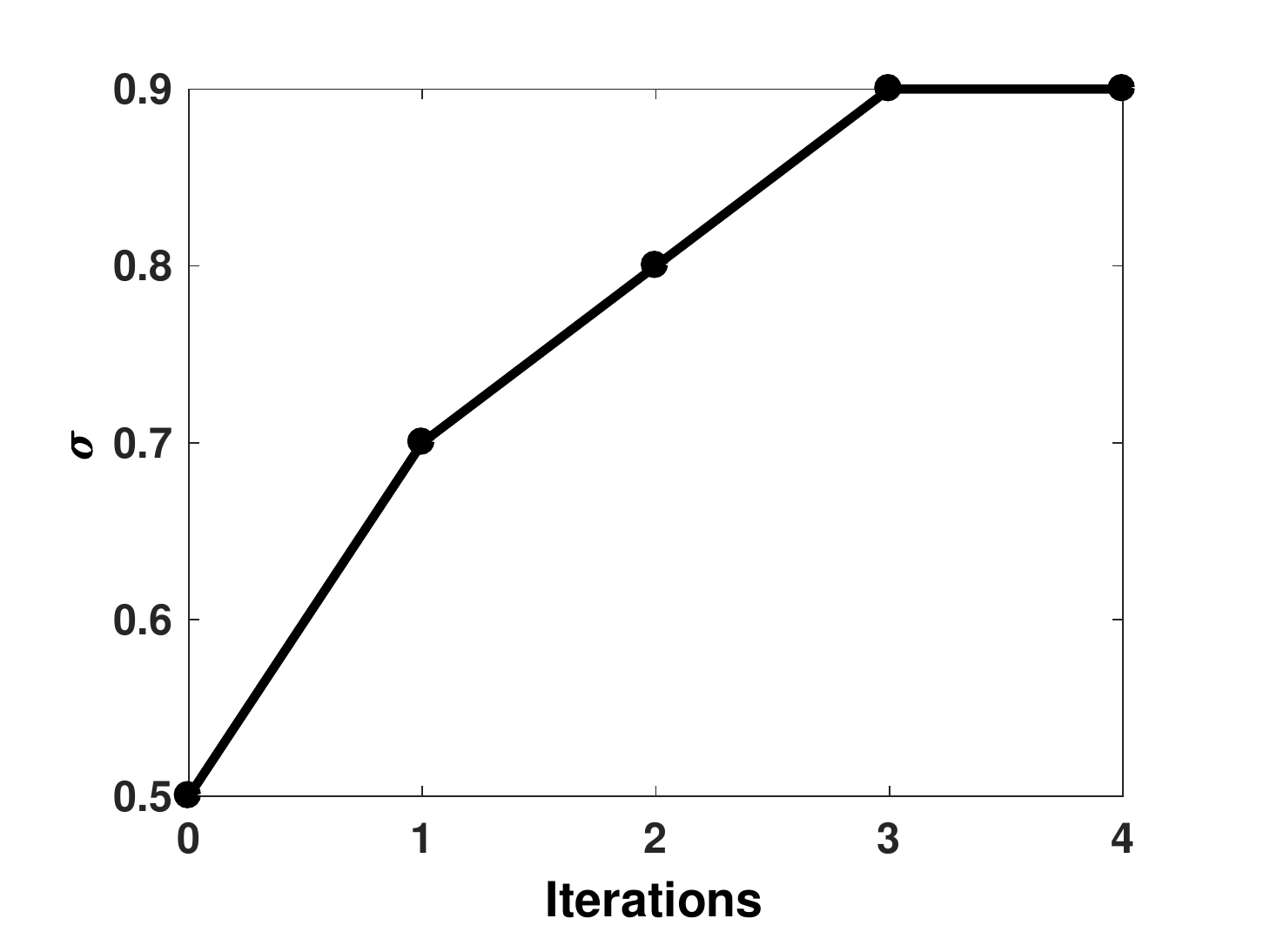}}
       \end{center}
 \caption{Algorithm performance throughout iterations when initialized with a low RBF kernel scale. (a) Convergence of the objective function (b) Stabilization of the misclassification rate (c) Stabilization of the RBF kernel scale $\sigma$ }
 \label{fig:exp_alg_conv_fei_low}
\end{figure}

In this first experiment, we study the iterative optimization procedure employed in the proposed method. As discussed in Section \ref{ssec:proposed_method}, the NSSE algorithm follows an alternating optimization scheme by minimizing the objective function in \eqref{eq:overall_obj} first with respect to the embedding $Y$ of the training samples, and then the scale parameter $\sigma$ of the RBF kernels. 

The results given in Figure \ref{fig:exp_alg_conv_fei_high} are obtained on the FEI face data set, where an embedding into a $d=10$ dimensional domain is computed using a total of $100$ training samples. Figure \ref{fig:obj_vs_iters_fei_high} shows the variation of the objective function in \eqref{eq:overall_obj} throughout the iterations. Although the proposed alternating optimization procedure is not theoretically guaranteed to find the global optimum of the objective, it is observed from the figure that the proposed scheme can effectively minimize the objective function, which converges in a small number of iterations. The misclassification rates of the test images in percentage are reported in Figure \ref{fig:msc_vs_iters_fei_high} obtained with the embeddings and interpolators computed in each iteration. The results show that the progressive update of the continuous embedding throughout the iterations improves the classification performance. The comparison of the plots in Figures \ref{fig:obj_vs_iters_fei_high} and \ref{fig:msc_vs_iters_fei_high} reveals that the variations of the objective function and the misclassification rate throughout the iterations are quite similar. This suggests that the choice of the objective function in \eqref{eq:overall_obj}, motivated by theoretical bounds, indeed matches the actual classification error. Figure \ref{fig:sgm_vs_iters_fei_high} shows the evolution of the RBF kernel scale parameter $\sigma$  throughout the iterations. The RBF kernel scale  $\sigma$ is deliberately initialized with a too high value in this experiment in order to study the effect of the initial conditions on the algorithm performance. Despite the initialization of $\sigma$ with a too large value, the iterative minimization of the objective gradually pulls the kernel scale towards a favorable value that improves the classification performance.

The same experiment is repeated in Figure \ref{fig:exp_alg_conv_fei_low}, by initializing the RBF kernel scale this time with a small value. It is observed that the RBF scale $\sigma$ is effectively optimized throughout the iterations towards a larger value, which gradually decreases the objective function and improves the classification accuracy. These results suggest that the algorithm performance is not affected much by the initialization of the RBF kernel scale. We have obtained similar results on the other data sets and under different choices of the parameters such as the number of training samples, which we skip here for brevity.


\subsection{Variation of the classification performance with the embedding dimension}
\label{ssec:error_vs_dimension}

We now study the classification performance of the proposed algorithm in relation with the dimension $d$ of the embedding. The proposed NSSE method is compared to some other dimensionality reduction algorithms listed below.

\begin{itemize}

\item  The Supervised Laplacian Eigenmaps (SUPLAP) method proposed in \cite{Raducanu12} computes a nonlinear low-dimensional embedding of the training samples by minimizing the objective in \eqref{eq:supLap_formal}. We extend the embedding of the training samples given by the SUPLAP method to the whole space via an RBF interpolator of the same form as in NSSE.  We then embed the test samples into the low-dimensional domain with this interpolation function.

\item  The Local Fisher Discriminant Analysis (LFDA) method proposed in \cite{Sugiyama07} is a supervised manifold learning algorithm that computes a linear embedding by optimizing a Fisher-type cost with additional locality preservation objectives.

\item  The Local Discriminant Embedding method (LDE) \cite{LDE} is a manifold learning method that optimizes a similar objective as in the SUPLAP method; however, learns a linear projection.



\item  Linear Discriminant Analysis (LDA) is a classical dimensionality reduction technique that maximizes the between-class scatter while minimizing the within-class scatter.  



\end{itemize}

The dimensionality reduction methods are applied on training samples to compute a $d$-dimensional embedding, which is then used to classify test samples via nearest neighbor classification in the domain of embedding. The algorithms are evaluated for a range of $d$ values. The parameters of the other methods in comparison are adjusted to attain their best performance.

The variation of the misclassification rates of test samples in percentage with the dimension $d$ of the embedding is presented in Figures \ref{fig:error_vs_d_yale}, \ref{fig:error_vs_d_coil}, \ref{fig:error_vs_d_orl}, \ref{fig:error_vs_d_fei}, \ref{fig:error_vs_d_robotics}, and \ref{fig:error_vs_d_mitcbcl}, respectively for the Yale, COIL-20, ORL, FEI, ROBOTICS-CSIE and the MIT-CBCL databases. The number of training images used in the computation of the embeddings are indicated in the figure captions for each experiment, which are chosen proportionally to the total number of samples in the data set. The results are the average of 20 random realizations of the experiments with different training and test sets. Most of the tested methods are based on solving a generalized eigenvalue problem and the rank of the involved matrices may be different for each method depending on the number of training samples and the number of classes. Hence, the maximum possible dimension of the embedding may vary between different methods, as well as the best range of dimensions where the methods perform well. For this reason, the results on each data set are grouped into two figures with different $d$ ranges for better visual clarity.

\begin{figure}[t]
\begin{center}
     \subfigure[NSSE, SUPLAP and LDA ]
       {\label{fig:dvsE_tdn10_nsse_suplap_lda}\includegraphics[scale=0.34]{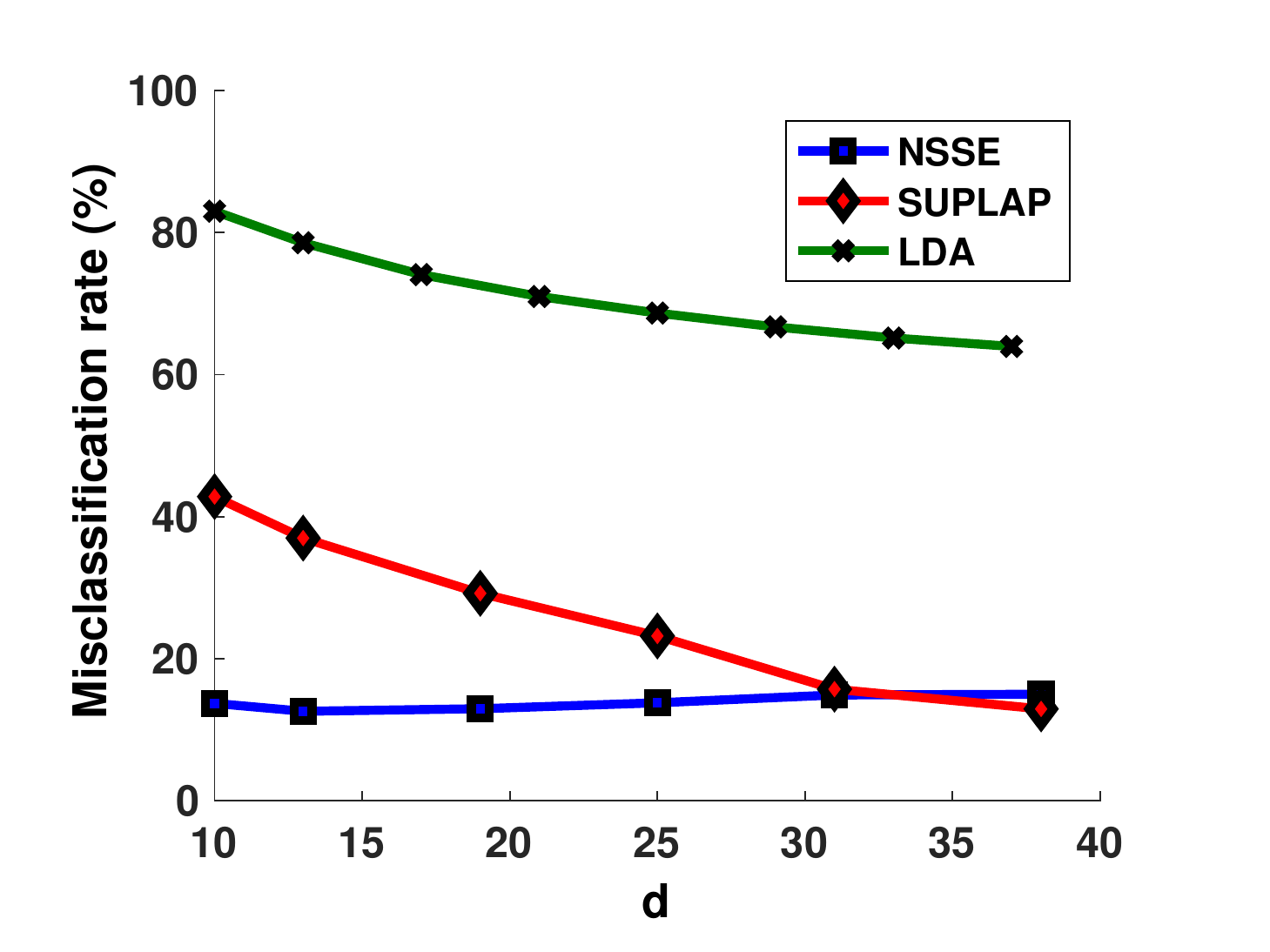}}
     \subfigure[LFDA and LDE ]
       {\label{fig:dvsE_tdn10_lfda_lde}\includegraphics[scale=0.34]{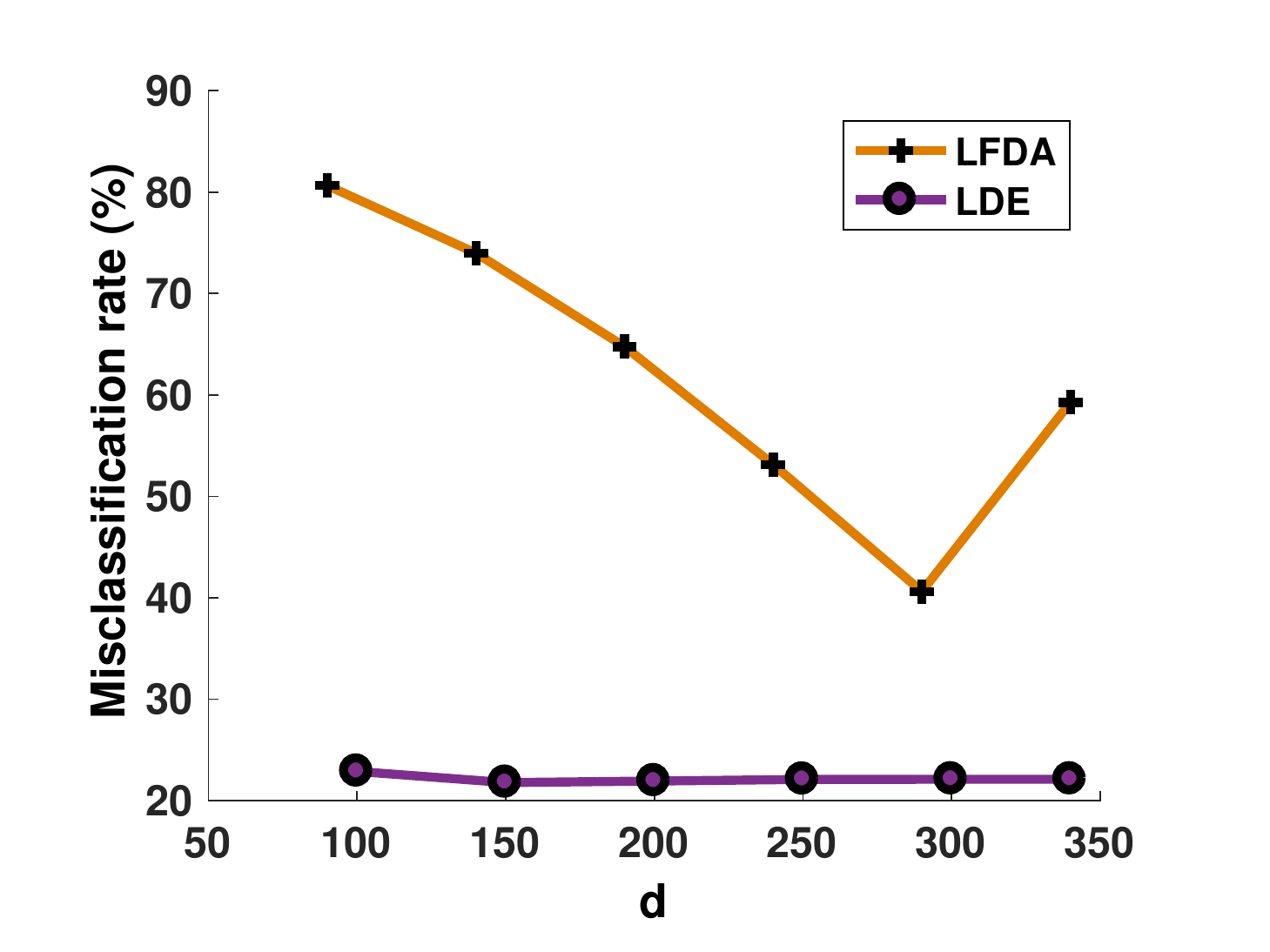}}
       \end{center}
        \vspace{-0.5cm}
 \caption{Variation of the misclassification rate with the embedding dimension in Yale data set, with 10 training samples per class}
 \label{fig:error_vs_d_yale}
\end{figure}

\begin{figure}[]
\begin{center}
     \subfigure[NSSE, SUPLAP and LDA]
       {\label{fig:dvsE_tdn10_nsse_suplap_lda}\includegraphics[scale=0.34]{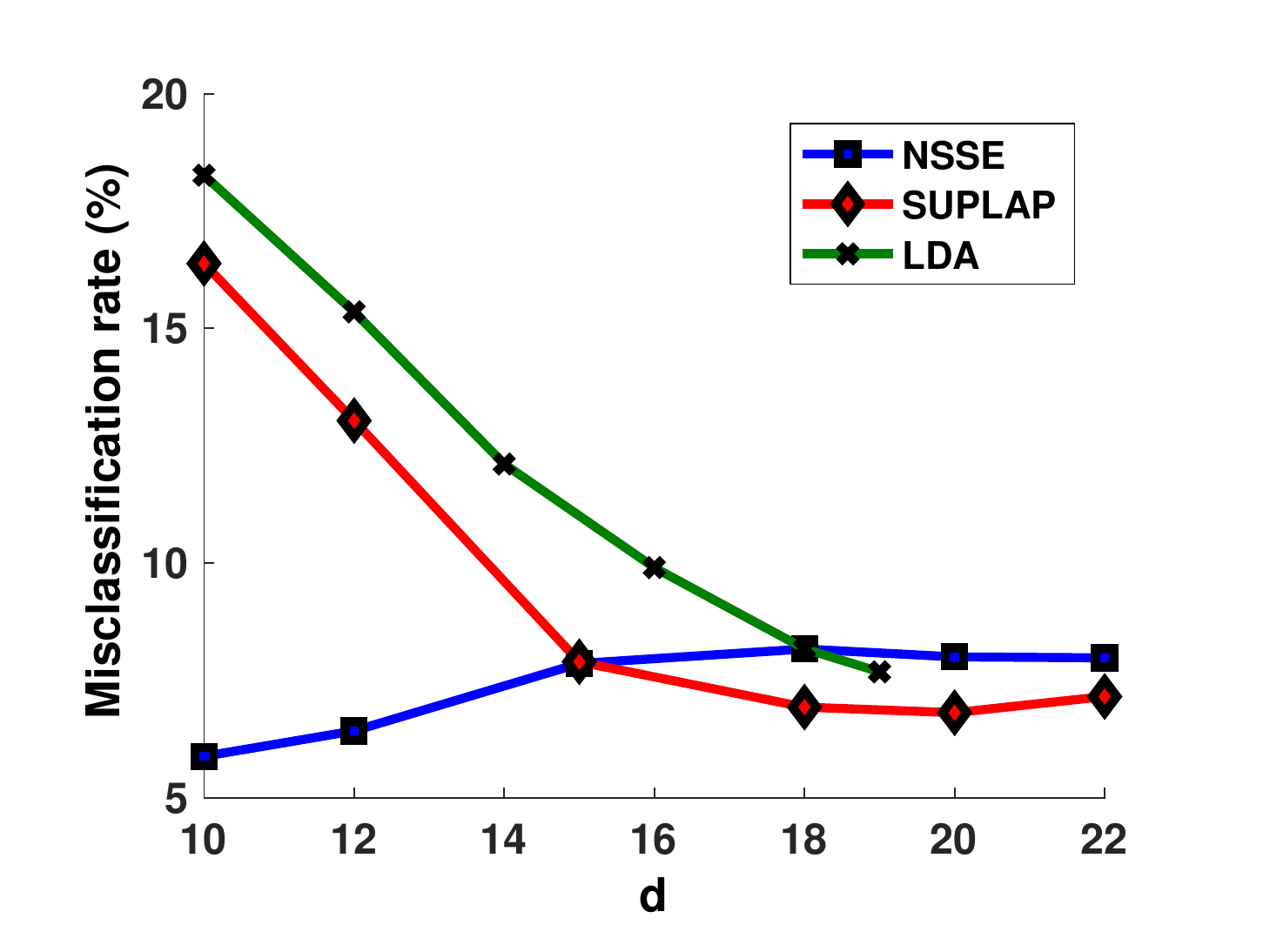}}
     \subfigure[LFDA and LDE]
       {\label{fig:dvsE_tdn10_lfda_lde}\includegraphics[scale=0.34]{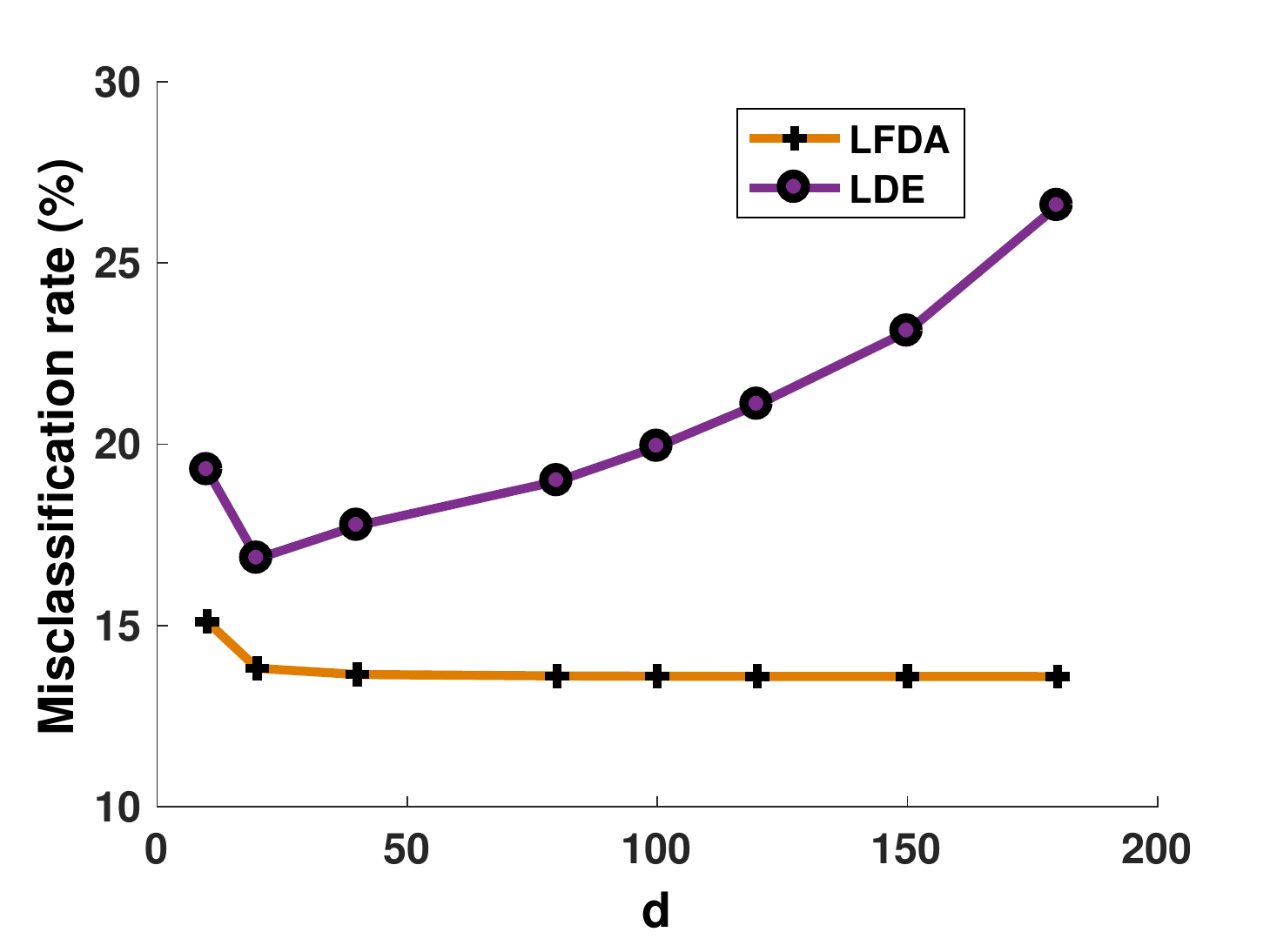}}
       \end{center}
       \vspace{-0.5cm}
 \caption{Variation of the misclassification rate with the embedding dimension in COIL-20 data set, with 10 training samples per class}
 \label{fig:error_vs_d_coil}
\end{figure}

\begin{figure}[]
\begin{center}
     \subfigure[NSSE, SUPLAP and LDA ]
       {\label{fig:dvsE_tdn2_all}\includegraphics[scale=0.34]{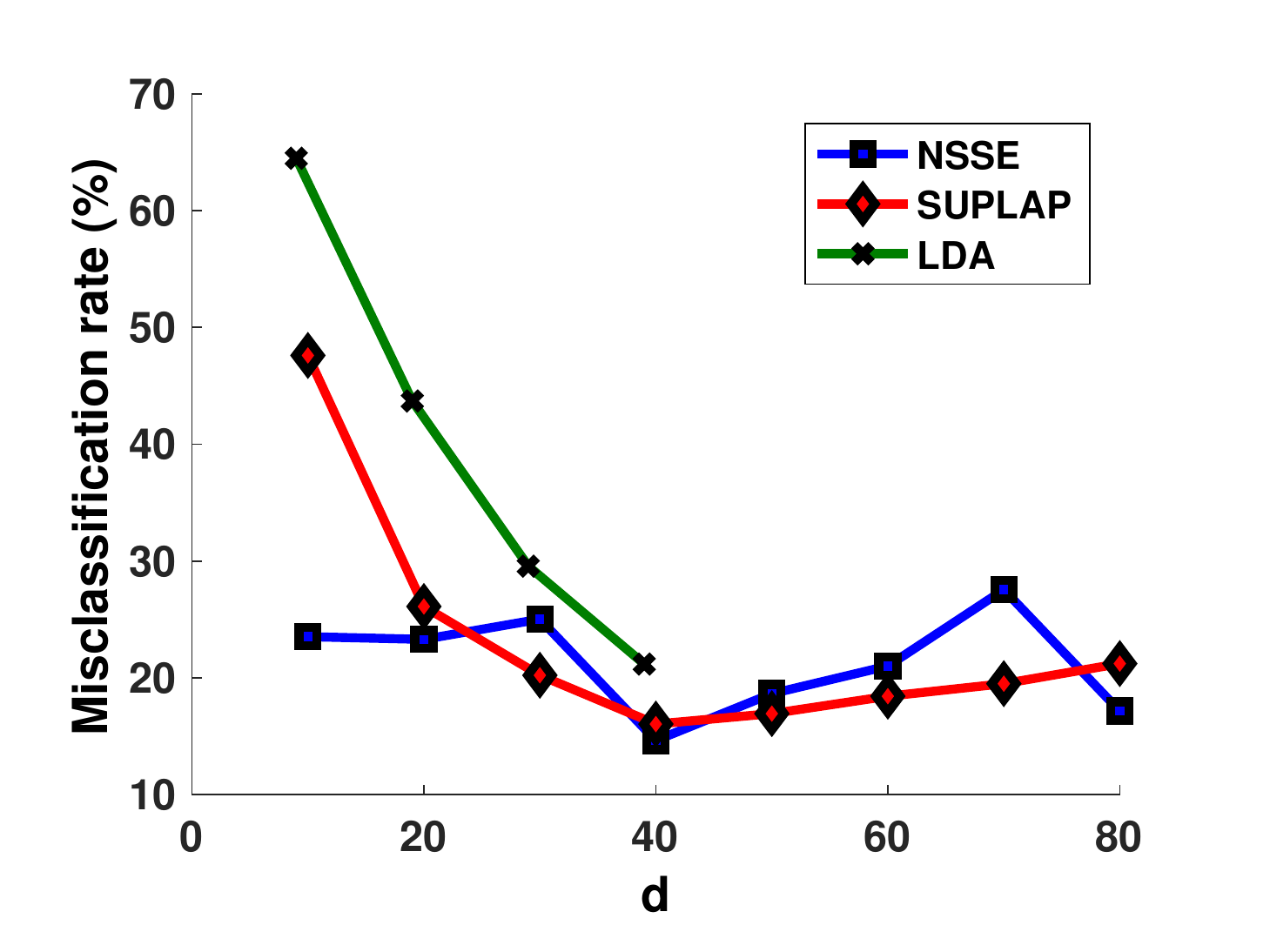}}
     \subfigure[LFDA and LDE ]
       {\label{fig:dvsE_tdn3_all}\includegraphics[scale=0.34]{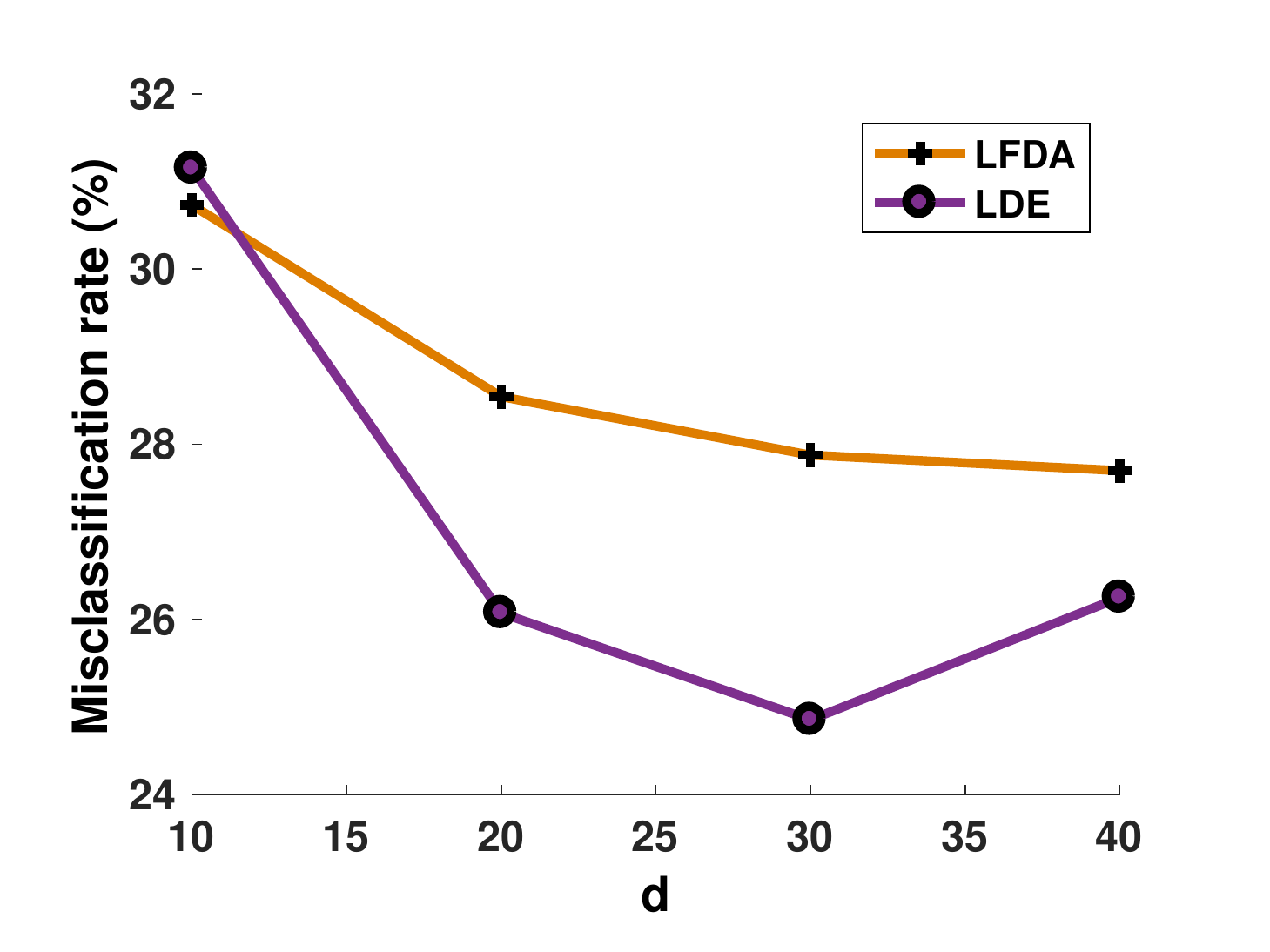}}
       \end{center}
            \vspace{-0.5cm}
 \caption{Variation of the misclassification rate with the embedding dimension in ORL data set, with 2 training samples per class}
 \label{fig:error_vs_d_orl}
\end{figure}

\begin{figure}[]
\begin{center}
     \subfigure[NSSE, SUPLAP and LDA]
       {\label{fig:dvsE_tdn2_nsse_suplap_lda}\includegraphics[scale=0.34]{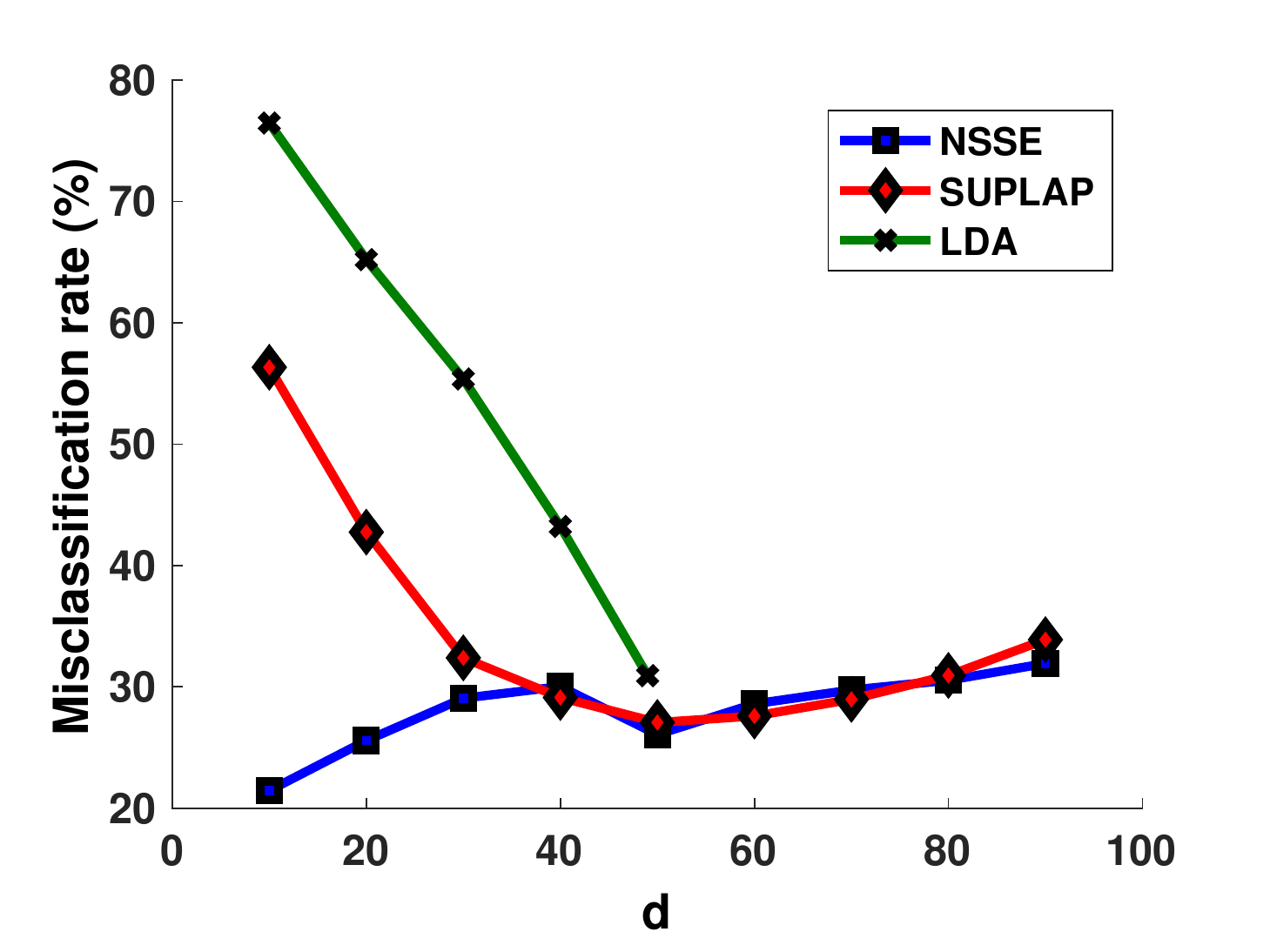}}
     \subfigure[LFDA and LDE]
       {\label{fig:dvsE_tdn2_lfda_lde}\includegraphics[scale=0.34]{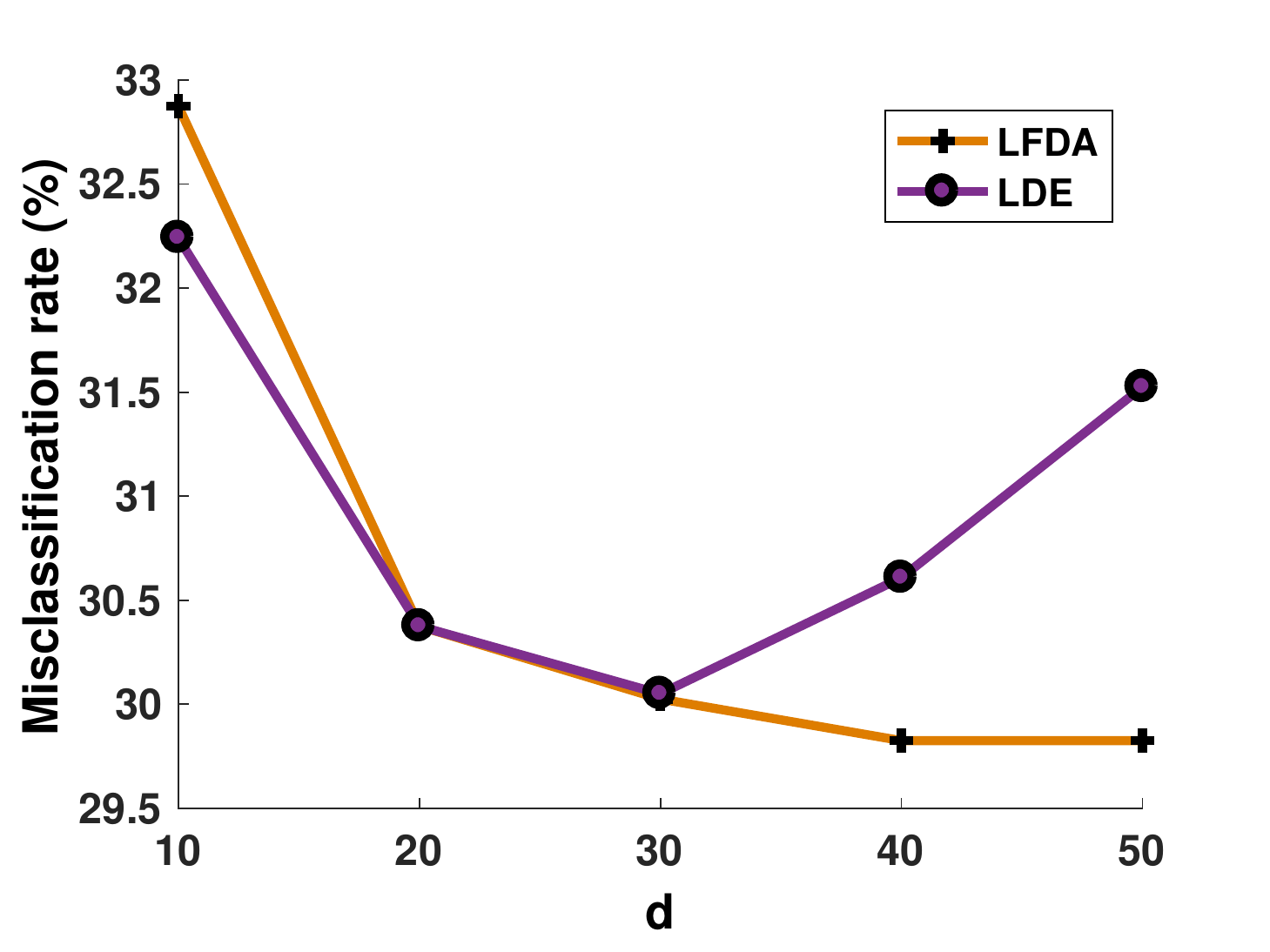}}
       \end{center}
       \vspace{-0.5cm}
 \caption{Variation of the misclassification rate with the embedding dimension in FEI data set, with 2 training samples per class}
 \label{fig:error_vs_d_fei}
\end{figure}

\begin{figure}[]
\begin{center}
     \subfigure[NSSE, SUPLAP and LDA]
       {\label{fig:dvsE_tdn7_nsse_suplap_lda}\includegraphics[scale=0.34]{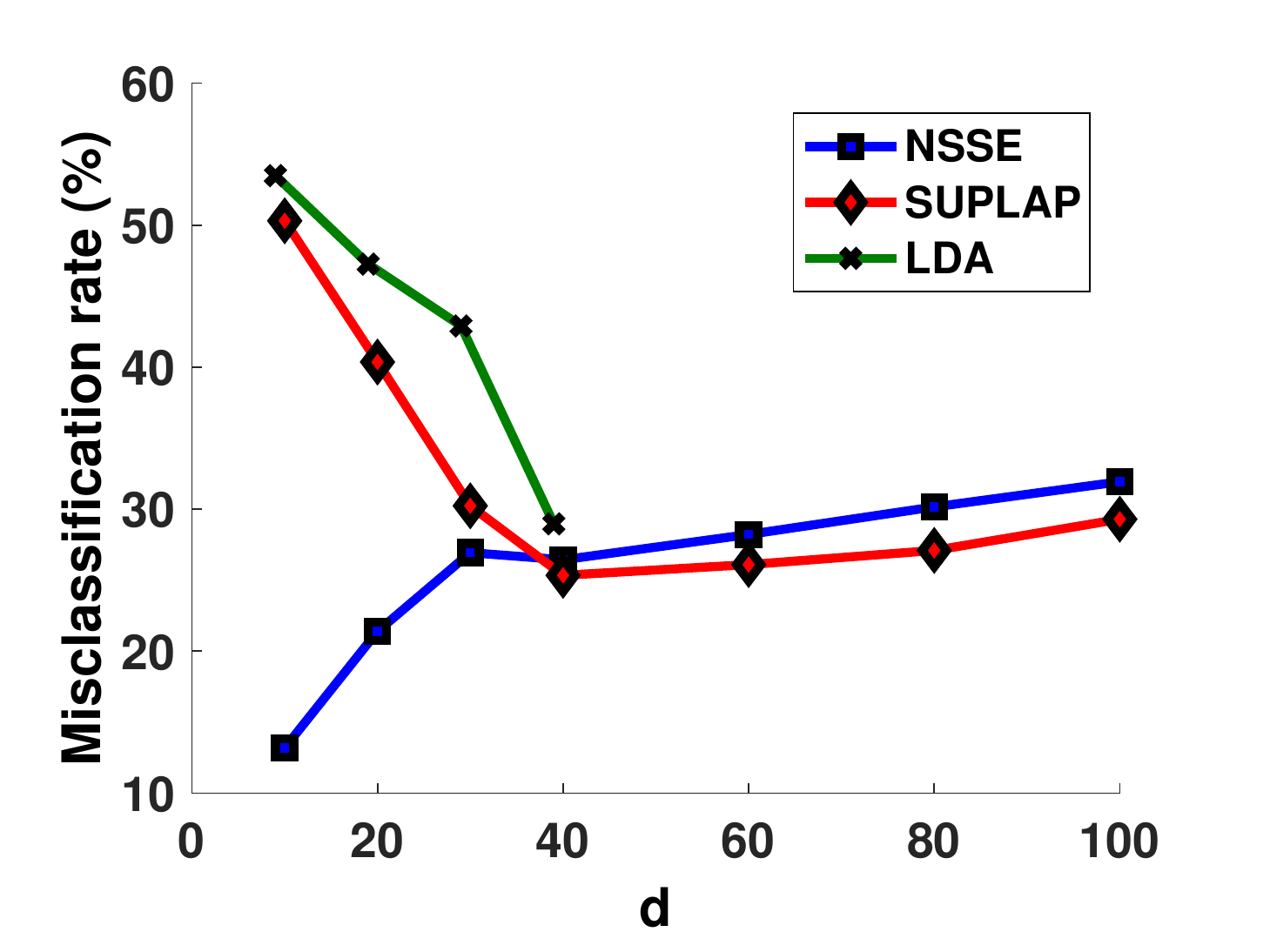}}
     \subfigure[LFDA and LDE]
       {\label{fig:dvsE_tdn7_lfda_lde}\includegraphics[scale=0.34]{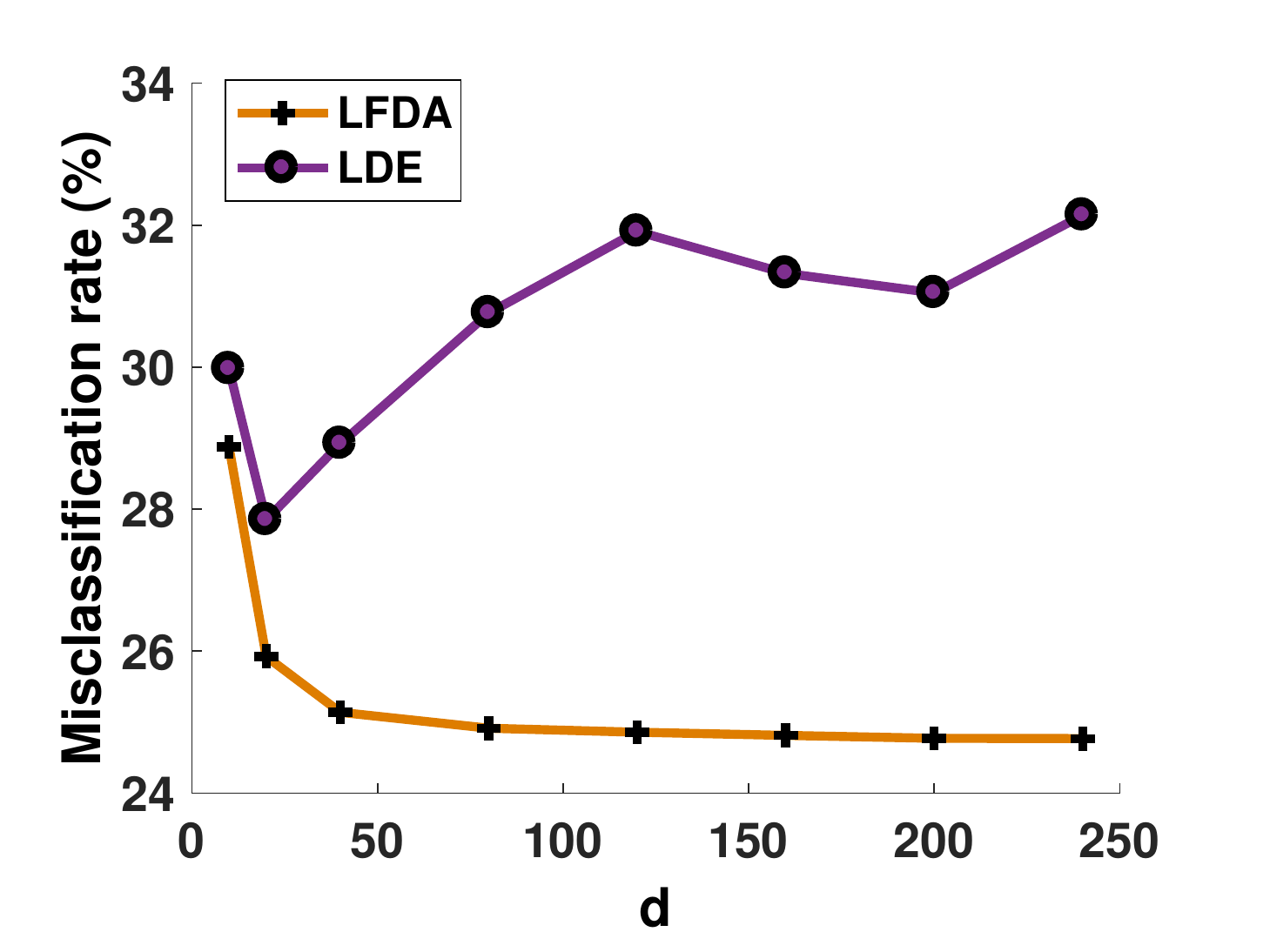}}
       \end{center}
        \vspace{-0.5cm}
 \caption{Variation of the misclassification rate with the embedding dimension in ROBOTICS-CSIE data set, with 7 training samples per class}
 \label{fig:error_vs_d_robotics}
\end{figure}

\begin{figure}[]
\begin{center}
     \subfigure[NSSE, SUPLAP and LDA]
       {\label{fig:dvsE_tdn10_nsse_suplap_lda}\includegraphics[scale=0.34]{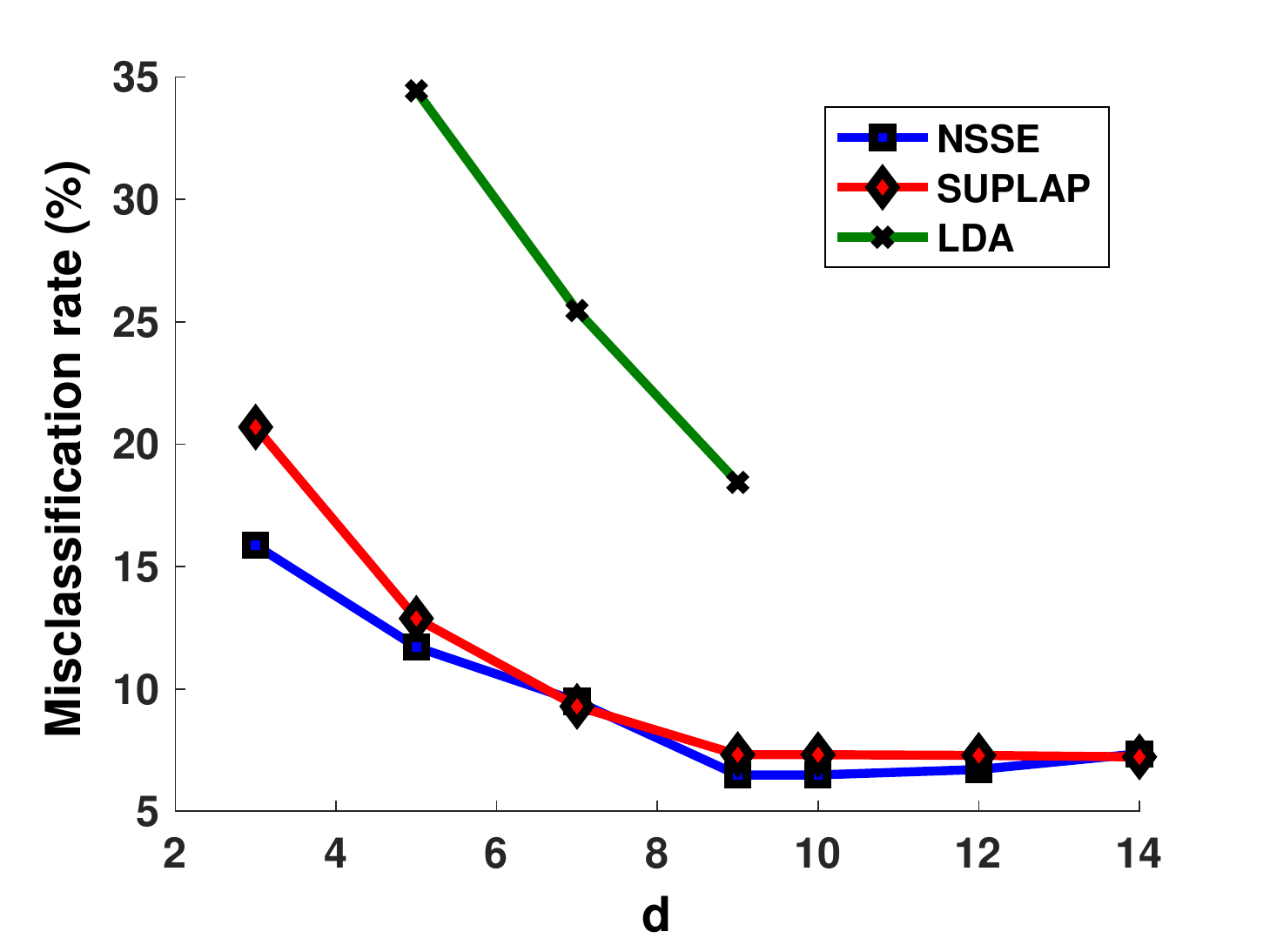}}
     \subfigure[LFDA and LDE]
       {\label{fig:dvsE_tdn10_lfda_lde}\includegraphics[scale=0.34]{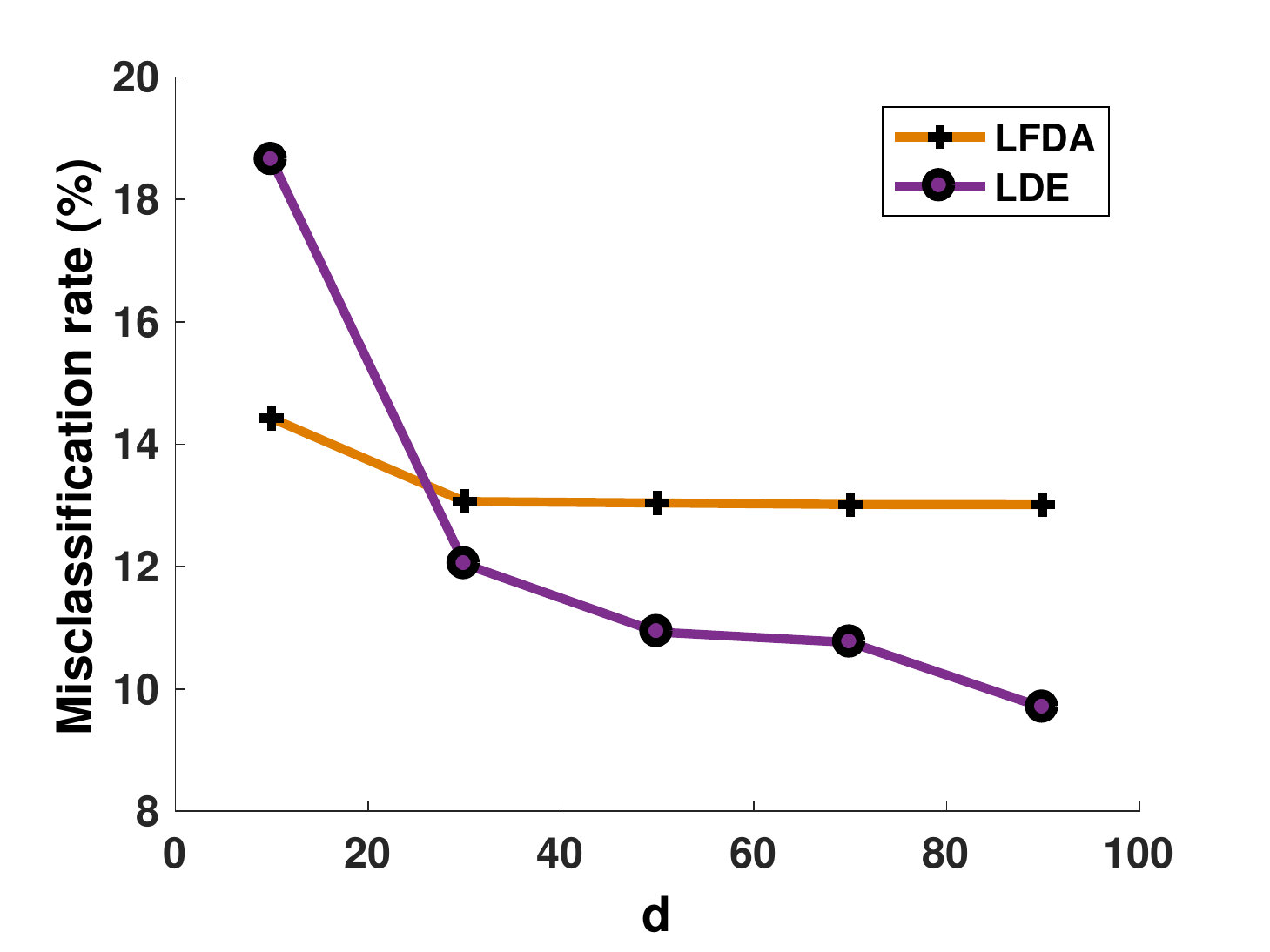}}
       \end{center}
        \vspace{-0.5cm}
 \caption{Variation of the misclassification rate with the embedding dimension in MIT-CBCL data set, with 10 training samples per class}
 \label{fig:error_vs_d_mitcbcl}
\end{figure}

The results in Figures \ref{fig:error_vs_d_yale}-\ref{fig:error_vs_d_mitcbcl} show that the classification accuracy of the proposed NSSE algorithm compares quite favorably to those of the other methods, as NSSE often yields the smallest misclassification rate at the optimal dimension. The misclassification rate of LDA is observed to decrease monotonically with the dimension $d$ and its best performance is attained when $d$ reaches the number of classes. The LDE and LFDA algorithms exhibit their best performances at much higher dimensions compared to the other algorithms. The error rates of these algorithms usually decrease as the embedding dimension increases; however, in some datasets a local optimum for $d$ can also be observed.

Among all methods, the nonlinear NSSE and SUPLAP methods often perform better than the linear LDA, LFDA, and LDE methods. This shows that the flexibility of nonlinear methods when learning an embedding is likely to bring an advantage in computing better representations for data. It is then interesting to compare the performances of the two nonlinear methods; NSSE and SUPLAP. The SUPLAP algorithm attains its best performance when the  dimension $d$ of the embedding is close to the number of classes, while the optimum value of $d$ for the proposed NSSE algorithm is smaller in most data sets. Interestingly, the optimal dimension of NSSE is much smaller than that of SUPLAP in data sets with a low intrinsic dimension such as COIL-20, FEI, and ROBOTICS-CSIE, which are generated by the variation of only one or two camera angle parameters. Similarly, in data sets of larger intrinsic dimension such as MIT-CBCL due to several pose and lighting parameters, the optimal dimension of NSSE is higher and closer to that of SUPLAP. This may suggest that the embedding computed with NSSE tries to capture the intrinsic geometry of data and provides a better representation when the embedding dimension is chosen proportionally to the intrinsic dimension of data.

The reduction of the embedding dimension is desirable especially regarding the complexity of the classification of test samples in a practical application. Another advantage of NSSE over SUPLAP is that NSSE is less sensitive to the choice of the dimension, as the misclassification performance is less affected for non-optimal values of $d$. Such benefits of the proposed NSSE algorithm mainly result from the fact that the Lipschitz continuity of the interpolator is imposed in the learning objective. Consequently, the training samples are embedded more evenly in the low-dimensional space so as to allow the construction of a regular interpolator, which in return reduces the required number of dimensions or the sensitivity to the non-optimal choice of $d$. 

\begin{figure}[t]
\begin{center}
     \subfigure[]
       {\label{fig:nsse_train_embed}\includegraphics[scale=0.3]{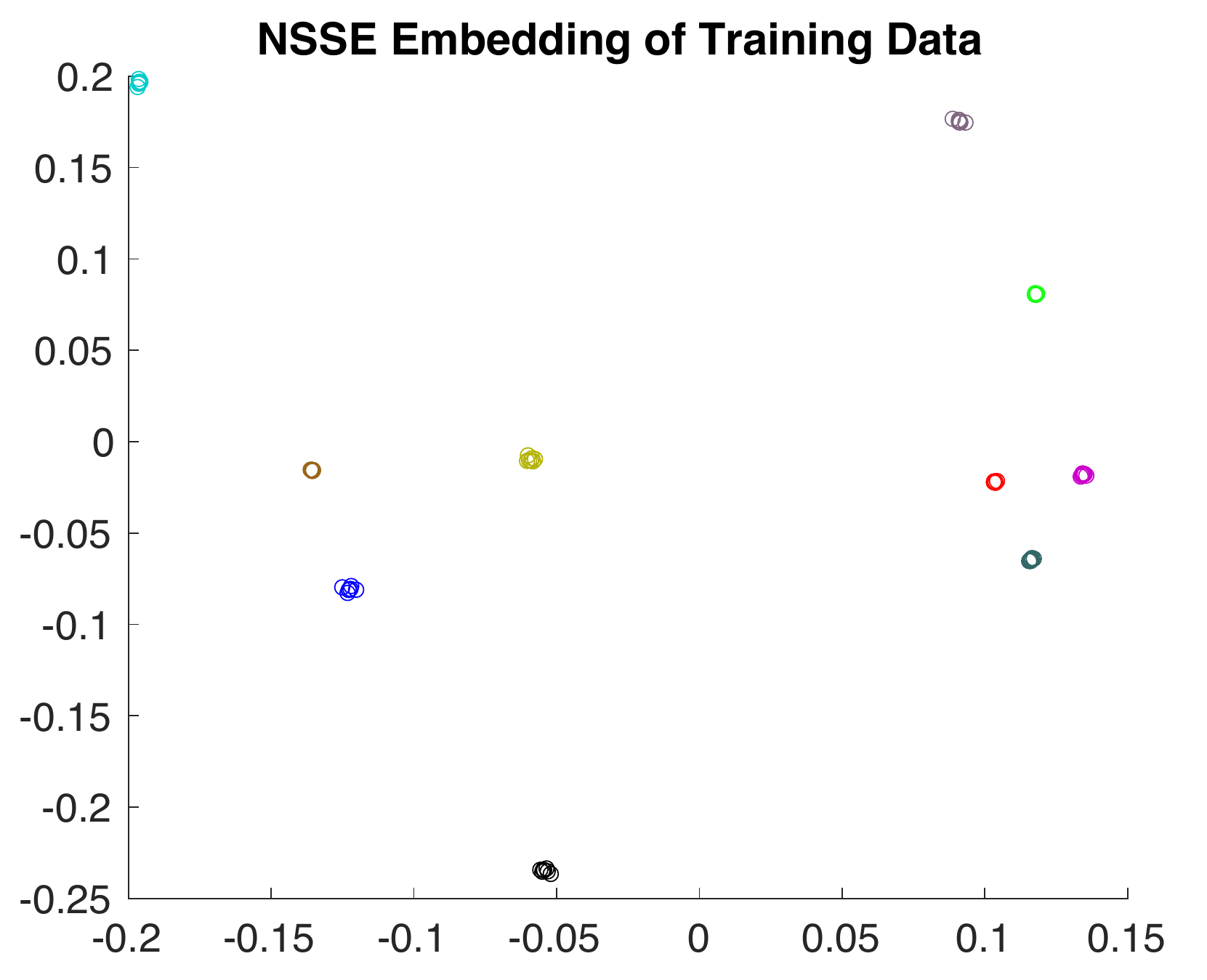}}
     \subfigure[]
       {\label{fig:suplap_train_embed}\includegraphics[scale=0.3]{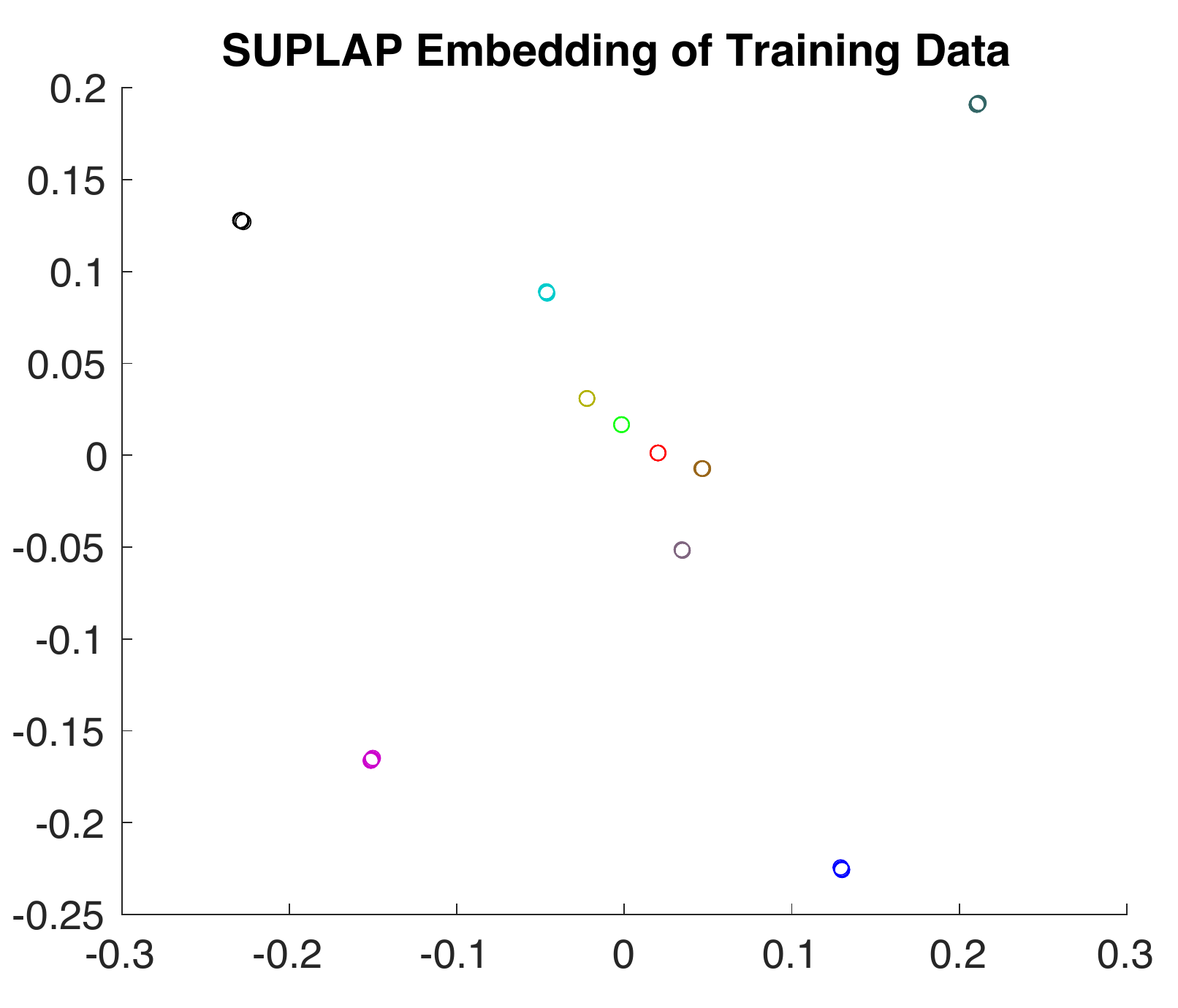}}
      \subfigure[]
       {\label{fig:nsse_test_embed}\includegraphics[scale=0.3]{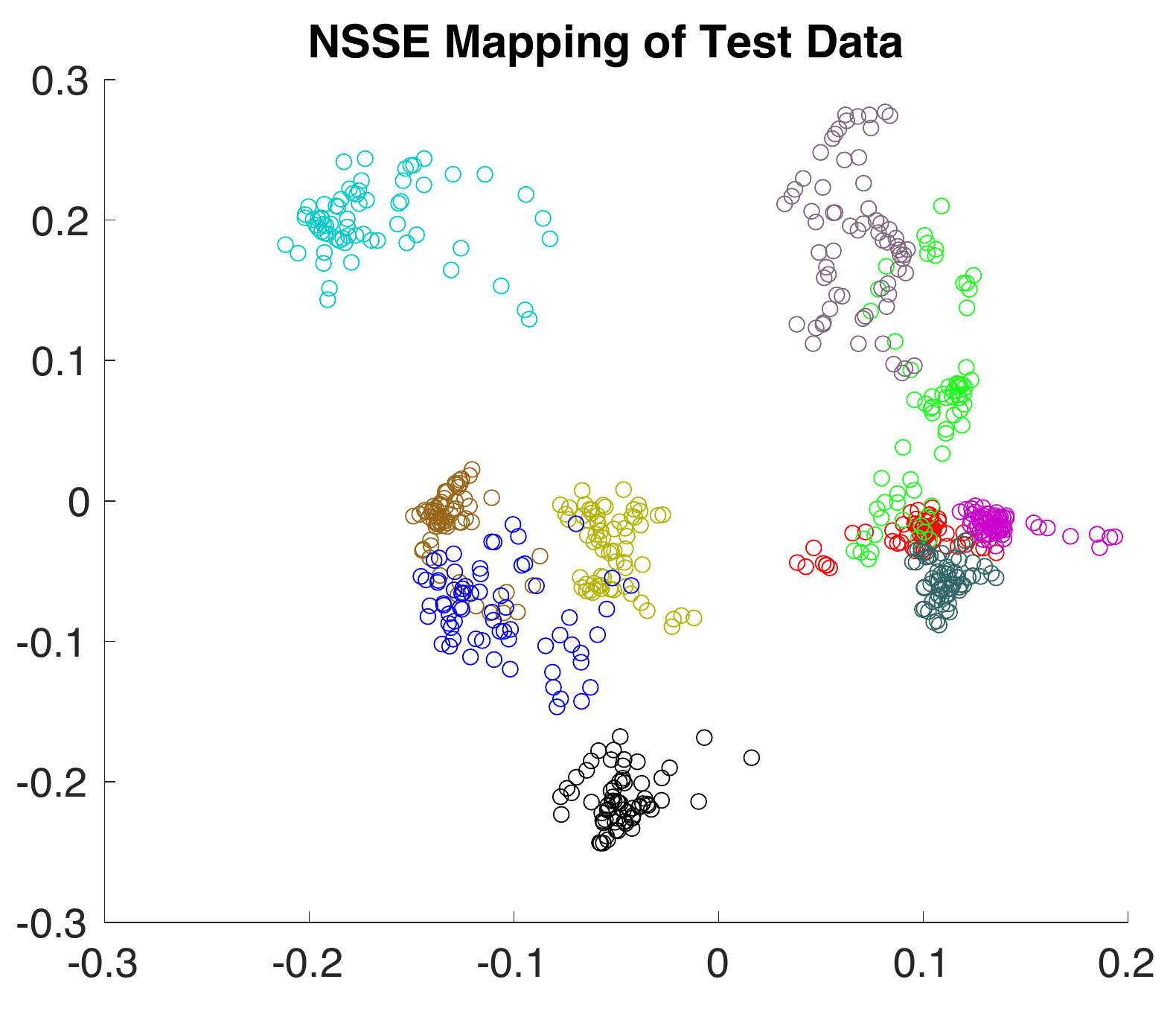}}
     \subfigure[]
       {\label{fig:suplap_test_embed}\includegraphics[scale=0.3]{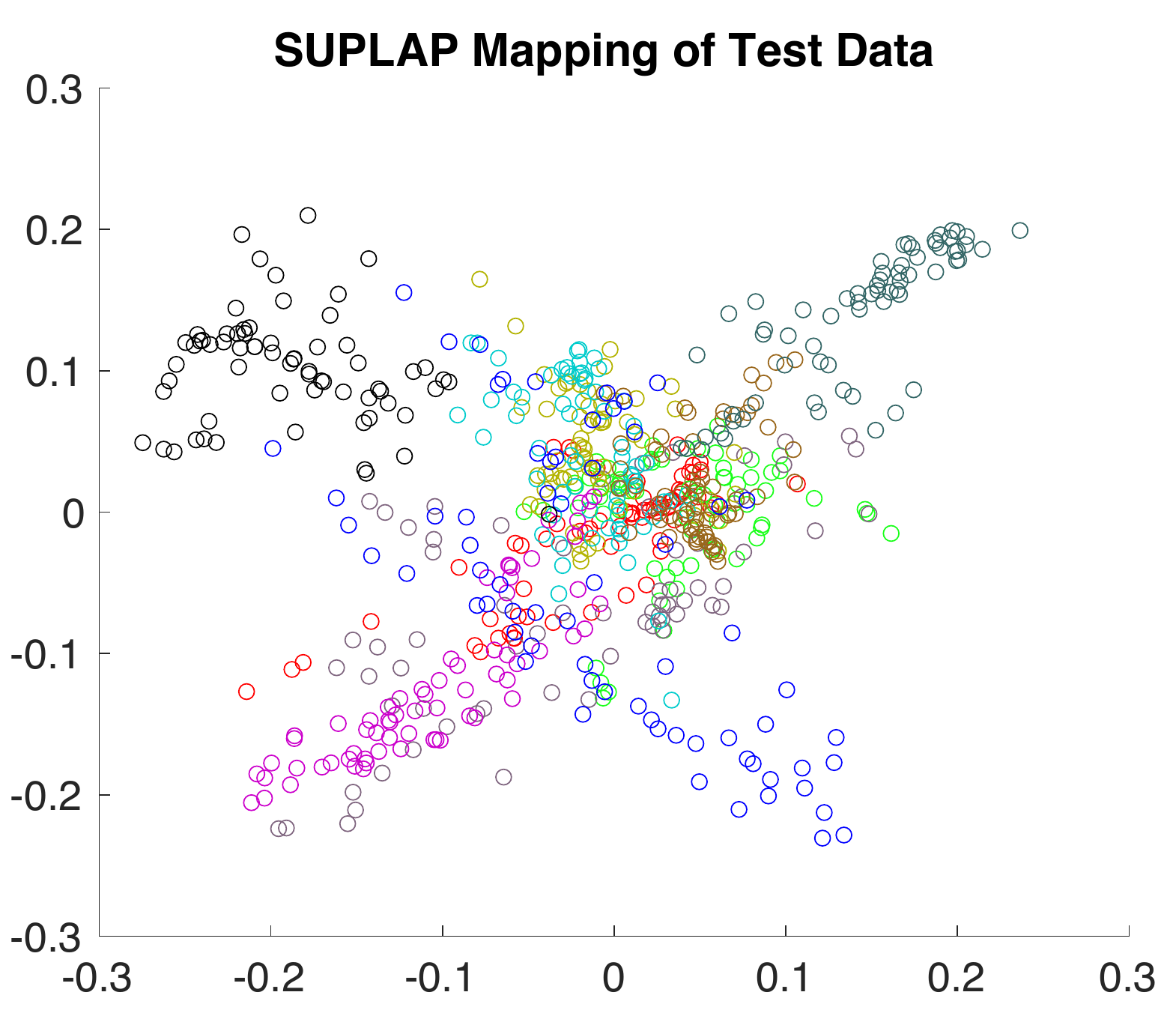}}
       \end{center}
        \vspace{-0.5cm}
 \caption{Visual comparison of the embeddings given by the NSSE and SUPLAP algorithms}
 \label{fig:visual_nsse_suplap_embed}
\end{figure}

In fact, Figure \ref{fig:visual_nsse_suplap_embed} provides a visual comparison of the embeddings obtained with the NSSE and the SUPLAP algorithms. Panels (a) and (b) show the two-dimensional embeddings of 70 training samples from 10 classes of the ROBOTICS-CSIE data set, respectively with the NSSE and the SUPLAP methods. The embeddings of training samples look similar between the two methods, although different classes are more regularly spaced in NSSE. The performance difference between these two methods becomes much clearer when the embeddings of the test samples in panels (c) and (d) are observed. Even at this very small embedding dimension of 2, the NSSE method separates test samples from different classes much more successfully than SUPLAP, which is due to the inclusion of the interpolator parameters in the learning objective in order to attain good generalization performance.

\subsection{Overall comparison with baseline classifiers and supervised dimensionality reduction methods}
\label{ssec:copm}

We now provide an overall comparison of the proposed NSSE method with baseline classifiers and other manifold learning methods. In addition to the supervised manifold learning algorithms used in the experiments of Section \ref{ssec:error_vs_dimension}, we compare NSSE with the SVM classifier in the original domain,  the nearest neighbor (NN) classifier in the original domain, the out-of-sample generalization of the Laplacian eigenmaps embedding with the Nystr\"om method (NYS) \cite{Bengio04}, the out-of-sample generalization of Laplacian eigenmaps with sparse coding (SPE) \cite{DornaikaR13}, the IsoKRR method proposed in \cite{OrsenigoV12} which computes a supervised nonlinear embedding and generalizes it with kernel ridge regression, and the supervised manifold regularization algorithm (MReg)  proposed in \cite{BelkinNS06} based on Reproducing Kernel Hilbert Spaces. The embedding dimensions and other algorithm parameters of the manifold learning methods are set to their optimal values yielding the best performance. The classification errors over test samples are studied by varying the training/test ratio and the results are averaged over 20 realizations of the experiments under different random choices of the training and test sets. 

The misclassification rates of test samples in percentage are presented for the compared methods for different training data sizes in Tables \ref{tab:dvsErr_yale}, \ref{tab:dvsErr_coil20}, \ref{tab:dvsErr_orl}, \ref{tab:dvsErr_fei}, \ref{tab:dvsErr_csie},
\ref{tab:dvsErr_mitcbcl}, respectively for the Yale, COIL-20, ORL, FEI, ROBOTICS-CSIE, and MIT-CBCL data sets. The leftmost columns of the tables show the number of training samples per class used for learning the classifiers. Experiments are conducted over a suitable range of number of training samples for each data set, considering the total number of samples in the data set. The smallest classification error of each experiment is shown in bold.

The proposed NSSE method is observed to outperform the other methods in most data sets. In Table \ref{tab:dvsErr_yale}, out-of-sample generalization with the Nystr\"om method NYS \cite{Bengio04} is seen to be one of the two best performing methods along with MReg \cite{BelkinNS06}, while its performance is behind many others in the other data sets. The extreme illumination changes in the Yale data set lead to degeneracies in the data manifold due to the very high local curvatures and non-differentiability, which seems to pose a challenge for the proposed NSSE method. Meanwhile, the global structure of this data set can in fact be approximated with linear subspace models fairly well, thanks to which an unsupervised out-of-sample extension method such as Nystr\"om achieves good performance on this data. In Tables \ref{tab:dvsErr_coil20}-\ref{tab:dvsErr_mitcbcl}, the proposed NSSE method is seen to yield the best performance in most settings, and the algorithms closest in performance to NSSE are the nonlinear and supervised SUPLAP, IsoKRR, and the MReg methods. Among the supervised manifold learning algorithms, the nonlinear methods seem to outperform the linear ones in general. The linear manifold learning algorithms LFDA, LDA, and LDE exhibit variable performance depending on the data set. As the performances of the algorithms improve with the increase in the number of training samples, these linear manifold learning methods may get outperformed by the baseline SVM and NN classifiers especially when the number of samples is sufficiently high. The performance gap between NSSE and the other nonlinear and supervised MReg, IsoKRR, and SUPLAP methods is more significant in the FEI and ROBOTICS-CSIE datasets containing a large number of classes, especially when the number of training samples is limited. The lack of training samples compared to the large number of classes is likely to lead to degenerate embeddings in nonlinear methods computing a pointwise embedding as in SUPLAP, while the regularization term enforcing the regularity of the interpolator in NSSE proves effective for the prevention of such degeneracies and ensuring the preservation of the overall geometric structure of data in the embedding.

Note that, unlike complex classifiers involving rich models with many parameters to learn, the classifiers obtained with the proposed method consist of a relatively simpler model with fewer parameters to learn. Based on models particularly fit to the priors on the data geometry and dimensionality, the proposed method attains satisfactory classification accuracy on data sets conforming to such low-dimensional models, even when the number of training samples is very limited. The accuracy of the proposed method would inevitably degrade if applied directly to data collections registered under highly uncontrolled settings violating the low-dimensional manifold assumption, e.g., data sets of complex backgrounds, with many different and dissimilar objects belonging to the same class, etc. Nevertheless, the learning of representations that extract the useful and essential information from such data sets registered under challenging conditions is still an open problem, and the proposed method can potentially be coupled with progressing representation learning techniques that can capture the geometric structure of data invariantly to its acquirement conditions.

\begin{table}[!h]
\footnotesize
\begin{tabular}{cccccccccccc}
\hline
 \# Tr & NSSE & SUPLAP & SVM & NN & LFDA & LDA & LDE & NYS &  SPE & IsoKRR & MReg \\
\hline
6 & 22.09 & 23.10 & 29.43 & 63.48 & 26.89 & 35.86 & 20.05 & \bf{19.47} & 63.81 & 22.58 & 19.89 \\
10 & 12.60 & 12.92 & 15.17 & 52.89 & 40.63 & 63.97 & 21.79 & \bf{11.88} & 53.19 & 12.58 & 12.36  \\
15 & 7.52 & 7.95 & 9.09 & 43.57 & 10.78 & 57.87 & 7.95 & 7.84 & 43.41 & 7.11 & \bf{6.90} \\
20 & 5.02 & 5.60 & 6.14 & 37.51 & 7.42 & 52.80 & 5.16 & 6.43 & 38.31 & 4.61 & \bf{4.50} \\
30 & 2.56 & 2.57 & 2.99 & 30.13 & 3.22 & 46.43 & 3.04 & 4.63 & 32.46 & 2.38 & \bf{2.35} \\
\hline
\end{tabular}
\caption{Misclassification rates (\%) of compared methods on Yale database} \label{tab:dvsErr_yale}
\end{table}

\begin{table}[!h]
\footnotesize
\begin{tabular}{cccccccccccc}
\hline
 \# Tr & NSSE & SUPLAP & SVM & NN & LFDA & LDA & LDE & NYS &  SPE & IsoKRR & MReg\\
\hline
7 & \bf{8.09} & 10.97 & 10.38 & 13.90 & 17.93 & 11.84 & 20.86 & 13.87 & 13.90 & 11.17 &8.49 \\
10 & \bf{4.97} & 6.81 & 6.93 & 10.22 & 13.59 & 7.68 & 16.84 & 9.38 & 10.22 & 7.07 & 5.44\\
15 & \bf{2.79} & 3.85 & 4.60 & 6.88 & 11.32 & 4.22 & 14.01 & 5.70 & 6.88 & 3.96 & 3.05\\
20 & \bf{1.25} & 2.04 & 3.23 & 4.51 & 9.53 & 2.29 & 12.64 & 3.34 & 4.54 & 2.00 & 1.31\\
30 & \bf{0.53} & 0.80 & 2.27 & 2.31 & 7.08 & 0.99 & 13.28 & 1.56 & 2.44 & 0.79 & 0.73\\
\hline
\end{tabular}
\caption{Misclassification rates (\%) of compared methods on COIL-20 database} \label{tab:dvsErr_coil20}
\end{table}

\begin{table}[!h]
\footnotesize
\begin{tabular}{cccccccccccc}
\hline
 \# Tr  & NSSE & SUPLAP & SVM & NN & LFDA & LDA & LDE & NYS & SPE & IsoKRR & MReg\\
\hline
2 & \bf{14.11} & 16.04 & 19.74 & 19.34 & 27.70 & 21.18 & 24.92 & 17.03 & 19.34 & 14.81 & 14.85  \\
3 & \bf{8.00} & 9.49 & 10.70 & 12.96 & 14.89 & 13.13 & 12.74 & 11.06 & 12.96 & 8.63 & 8.38  \\
5 & \bf{3.90} & 5.32 & 4.35 & 6.92 & 8.10 & 7.74 & 7.05 & 6.90 & 6.92 & 4.23 & 4.13 \\
\hline
\end{tabular}
\caption{Misclassification rates (\%) of compared methods on ORL database} \label{tab:dvsErr_orl}
\end{table}

\begin{table}[!h]
\footnotesize
\begin{tabular}{cccccccccccc}
\hline
 \# Tr  & NSSE & SUPLAP & SVM & NN & LFDA & LDA & LDE & NYS &  SPE & IsoKRR & MReg \\
\hline
2 & \bf{20.86} & 27.07 & 35.38 & 32.13 & 29.83 & 30.93 & 30.05 & 31.91 & 32.13 & 26.42 & 25.03 \\
4 & \bf{8.05} & 12.46 & 12.85 & 19.45 & 12.90 & 12.56 & 10.80 & 19.20 & 19.45 & 11.06 & 10.77 \\
7 & \bf{5.00} & 6.42 & 9.09 & 10.86 & 9.74 & 5.40 & 7.77 & 11.53 & 11.23 & 5.03 & 5.40 \\
\hline
\end{tabular}
\caption{Misclassification rates (\%) of compared methods on FEI database} \label{tab:dvsErr_fei}
\end{table}

\begin{table}[!h]
\footnotesize
\begin{tabular}{cccccccccccc}
\hline
 \# Tr & NSSE & SUPLAP & SVM & NN & LFDA & LDA & LDE & NYS &  SPE & IsoKRR & MReg\\
\hline
7 & \bf{13.56} & 27.23 & 23.97 & 34.46 & 24.87 & 29.43 & 25.13 & 34.87 & 34.53 & 24.29 & 20.32 \\
14 & \bf{4.38} & 11.74 & 8.78 & 17.80 & 11.97 & 14.15 & 9.74 & 17.36 & 17.84 & 9.04 & 5.86 \\
21 & \bf{2.83} & 6.52 & 4.77 & 10.09 & 6.99 & 10.57 & 5.88 & 9.85 & 9.76 & 4.81 & 3.08 \\
\hline
\end{tabular}
\caption{Misclassification rates (\%) of compared methods on ROBOTICS-CSIE database} \label{tab:dvsErr_csie}
\end{table}

\begin{table}[!h]
\footnotesize
\begin{tabular}{cccccccccccc}
\hline
 \# Tr  & NSSE & SUPLAP & SVM & NN & LFDA & LDA & LDE & NYS &  SPE & IsoKRR & MReg\\
\hline
10 & \bf{6.48} & 7.31 & 9.91 & 14.43 & 12.32 & 18.44 & 9.69 & 15.03 & 14.43 & 6.53 & 6.55 \\
20 & \bf{2.49} & 3.38 & 4.18 & 5.65 & 8.36 & 8.38 & 6.02 & 6.06 & 5.66 & 2.50 & 2.85 \\
40 & 0.77 & 1.22 & 1.52 & 1.46 & 5.29 & 3.18 & 2.97 & 1.84 & 2.05 & \bf{0.71} & 0.97 \\
\hline
\end{tabular}
\caption{Misclassification rates (\%) of compared methods on MIT-CBCL database} \label{tab:dvsErr_mitcbcl}
\end{table}

\begin{table}[!h]
\vspace{0cm}
\hspace{0cm}
\begin{tabular}{ |c|c|c|c|c|c| }
\hline
& COIL-20  & FEI & ROBOTICS-CSIE  \\
\hline
  & 140 | 0.81 sec & 100 | 0.82 sec  & 280 | 2.22 sec   \\
Data size | Running time & 300 | 1.45 sec & 200 | 1.52 sec & 560 | 8.33 sec \\
 & 600 | 5.45 sec & 350 | 3.74 sec & 840 | 14.99 sec \\
\hline 
\end{tabular}
\caption{Running times of the NSSE algorithm observed for several data sizes on three data sets. The data size stands for the total number of training images.}
\label{table:runtimes_NSSE}
\end{table}

Finally, we report the observed computation times for jointly learning an embedding and an interpolator with the proposed NSSE algorithm. The running times obtained for a single run of the NSSE algorithm with a non-optimized MATLAB implementation on a laptop computer are given in Table \ref{table:runtimes_NSSE} for the COIL-20, FEI and ROBOTICS-CSIE data sets, for different number of training images. The observed running times seem to be consistent with the complexity analysis of the method provided in Section \ref{ssec:comp_analysis}.


\section{Conclusion}
\label{sec:conclusion}

We have proposed a nonlinear supervised manifold learning method that learns an embedding of the training data jointly with a smooth RBF interpolation function that extends the embedding to the whole space. The embedding and the interpolator parameters are jointly optimized with the purpose of good generalization to initially unavailable data, based on recent theoretical results on the performance of supervised manifold learning algorithms. In particular, the embedding and the RBF paramaters are learnt such that the interpolator has sufficiently good Lipschitz regularity while the samples from different classes are separated as much as possible. Experiments on image data sets have shown that the proposed method learns embeddings often yielding better classification performance while requiring a smaller number of dimensions in comparison with other supervised manifold learning approaches. Thanks to the priors on the Lipschitz regularity of the interpolator incorporated in the learning objective, the proposed method can learn efficient representations even under limited availability of training samples, and is relatively robust to conditions such as the non-optimal choice of the embedding dimension and unfavorable initialization of the interpolator parameters. Our study shows that nonlinear mappings are promising in supervised dimensionality reduction, and taking into account the generalizability of the embedding explicitly in the learning objective highly improves the classification performance.

\footnotesize
\bibliography{mybibfile}

\begin{thebibliography}{10}
\expandafter\ifx\csname url\endcsname\relax
  \def\url#1{\texttt{#1}}\fi
\expandafter\ifx\csname urlprefix\endcsname\relax\def\urlprefix{URL }\fi
\expandafter\ifx\csname href\endcsname\relax
  \def\href#1#2{#2} \def\path#1{#1}\fi

\bibitem{Tenenbaum00}
J.~B. Tenenbaum, V.~de~Silva, J.~C. Langford, A global geometric framework for
  nonlinear dimensionality reduction., Science 290~(5500) (2000) 2319--2323.

\bibitem{Roweis00}
S.~T. Roweis, L.~K. Saul, Nonlinear dimensionality reduction by locally linear
  embedding, Science 290 (2000) 2323--2326.

\bibitem{Belkin03}
M.~Belkin, P.~Niyogi, Laplacian eigenmaps for dimensionality reduction and data
  representation, Neural Computation 15~(6) (2003) 1373--1396.

\bibitem{He04}
X.~He, P.~Niyogi, {Locality Preserving Projections}, in: Advances in Neural
  Information Processing Systems 16, MIT Press, Cambridge, MA, 2004.

\bibitem{Donoho03}
D.~L. Donoho, C.~Grimes, {Hessian eigenmaps: Locally linear embedding
  techniques for high-dimensional data}, Proceedings of the National Academy of
  Sciences of the United States of America 100~(10) (2003) 5591--5596.

\bibitem{Zhang2005}
Z.~Zhang, H.~Zha, Principal manifolds and nonlinear dimension reduction via
  local tangent space alignment, SIAM Journal of Scientific Computing 26 (2005)
  313--338.

\bibitem{Sugiyama07}
M.~Sugiyama, Dimensionality reduction of multimodal labeled data by local
  fisher discriminant analysis, Journal of Machine Learning Research 8 (2007)
  1027--1061.

\bibitem{Hua12}
Q.~Hua, L.~Bai, X.~Wang, Y.~Liu, Local similarity and diversity preserving
  discriminant projection for face and handwriting digits recognition.,
  Neurocomputing 86 (2012) 150--157.

\bibitem{Yang11}
W.~Yang, C.~Sun, L.~Zhang, A multi-manifold discriminant analysis method for
  image feature extraction, Pattern Recognition 44~(8) (2011) 1649--1657.

\bibitem{Zhang12}
Z.~Zhang, M.~Zhao, T.~Chow, Marginal semi-supervised sub-manifold projections
  with informative constraints for dimensionality reduction and recognition,
  Neural Networks 36 (2012) 97--111.

\bibitem{Li13}
B.~Li, J.~Liu, Z.~Zhao, W.~Zhang, Locally linear representation fisher
  criterion, in: The 2013 International Joint Conference on Neural Networks,
  2013, pp. 1--7.

\bibitem{Cui12}
Y.~Cui, L.~Fan, A novel supervised dimensionality reduction algorithm:
  Graph-based fisher analysis, Pattern Recognition 45~(4) (2012) 1471--1481.

\bibitem{Wang09}
R.~Wang, X.~Chen, Manifold discriminant analysis, in: CVPR, 2009, pp. 429--436.

\bibitem{YuSZH16}
M.~Yu, L.~Shao, X.~Zhen, X.~He, Local feature discriminant projection, {IEEE}
  Trans. Pattern Anal. Mach. Intell. 38~(9) (2016) 1908--1914.

\bibitem{Raducanu12}
B.~Raducanu, F.~Dornaika, A supervised non-linear dimensionality reduction
  approach for manifold learning, Pattern Recognition 45~(6) (2012) 2432--2444.

\bibitem{Bengio04}
Y.~Bengio, J.~F. Paiement, P.~Vincent, O.~Delalleau, N.~Le~Roux, M.~Ouimet,
  Out-of-sample extensions for {LLE}, {ISOMAP}, {MDS}, {E}igenmaps, and
  {S}pectral {C}lustering, in: Adv. Neural Inf. Process. Syst., MIT Press,
  2004, pp. 177--184.

\bibitem{QiaoZWZ13}
H.~Qiao, P.~Zhang, D.~Wang, B.~Zhang, An explicit nonlinear mapping for
  manifold learning, {IEEE} T. Cybernetics 43~(1) (2013) 51--63.

\bibitem{ChenWG13}
G.~H. Chen, C.~Wachinger, P.~Golland, Sparse projections of medical images onto
  manifolds, in: Proc. 23rd Int. Conf. Information Processing in Medical
  Imaging, 2013, pp. 292--303.

\bibitem{PeherstorferPB11}
B.~Peherstorfer, D.~Pfl{\"{u}}ger, H.~J. Bungartz, A sparse-grid-based
  out-of-sample extension for dimensionality reduction and clustering with
  laplacian eigenmaps, in: Proc. 24th Australasian Joint Conf. Advances in
  Artificial Intelligence, 2011, pp. 112--121.

\bibitem{VuralG16}
E.~Vural, C.~Guillemot, Out-of-sample generalizations for supervised manifold
  learning for classification, IEEE Transactions on Image Processing 25~(3)
  (2016) 1410--1424.

\bibitem{VuralGTR}
E.~Vural, C.~Guillemot, A study of the classification of low-dimensional data
  with supervised manifold learning, Accepted for publication in Journal of
  Machine Learning Research. [Online]. Available:
  https://arxiv.org/abs/1507.05880.

\bibitem{LDE}
H.~Chen, H.~Chang, T.~Liu, Local discriminant embedding and its variants, in:
  {IEEE} Computer Society Conf. Computer Vision and Pattern Recognition, 2005,
  pp. 846--853.

\bibitem{MaronidisTP15}
A.~Maronidis, A.~Tefas, I.~Pitas, Subclass graph embedding and a marginal
  fisher analysis paradigm, Pattern Recognition 48~(12) (2015) 4024--4035.

\bibitem{GaoMZGL13}
Q.~Gao, J.~Ma, H.~Zhang, X.~Gao, Y.~Liu, Stable orthogonal local discriminant
  embedding for linear dimensionality reduction, {IEEE} Trans. Image Processing
  22~(7) (2013) 2521--2531.

\bibitem{LFDP}
M.~Yu, L.~Shao, X.~Zhen, X.~He, Local feature discriminant projection, {IEEE}
  Trans. Pattern Anal. Mach. Intell. 38~(9) (2016) 1908--1914.

\bibitem{ChenWLL17}
S.~Chen, J.~Wang, C.~Liu, B.~Luo, Two-dimensional discriminant locality
  preserving projection based on l1-norm maximization, Pattern Recognition
  Letters 87 (2017) 147--154.

\bibitem{ESLLE}
S.~Zhang, Enhanced supervised locally linear embedding, Pattern Recognition
  Letters 30~(13) (2009) 1208--1218.

\bibitem{NPDE}
Y.~Pang, A.~T.~B. Jin, F.~S. Abas, Neighbourhood preserving discriminant
  embedding in face recognition, J. Visual Communication and Image
  Representation 20~(8) (2009) 532--542.

\bibitem{HeCYZ05}
X.~He, D.~Cai, S.~Yan, H.~Zhang, Neighborhood preserving embedding, in: 10th
  {IEEE} Int. Conf. Computer Vision, 2005, pp. 1208--1213.

\bibitem{HyME}
Y.~Liu, Y.~Liu, K.~C.~C. Chan, K.~A. Hua, Hybrid manifold embedding, {IEEE}
  Trans. Neural Netw. Learning Syst. 25~(12) (2014) 2295--2302.

\bibitem{ZhouS17}
Y.~Zhou, S.~Sun, Manifold partition discriminant analysis, {IEEE} Trans.
  Cybernetics 47~(4) (2017) 830--840.

\bibitem{DornaikaR13}
F.~Dornaika, B.~Raducanu, Out-of-sample embedding for manifold learning applied
  to face recognition, in: {IEEE} Conf. Computer Vision and Pattern
  Recognition, {CVPR} Workshop, 2013, pp. 862--868.

\bibitem{OrsenigoV12}
C.~Orsenigo, C.~Vercellis, Kernel ridge regression for out-of-sample mapping in
  supervised manifold learning, Expert Syst. Appl. 39~(9) (2012) 7757--7762.

\bibitem{ScholkopfSM97}
B.~Sch{\"{o}}lkopf, A.~J. Smola, K.~R. M{\"{u}}ller, Kernel principal component
  analysis, in: 7th Int. Conf. Artificial Neural Networks, 1997, pp. 583--588.

\bibitem{BaudatA00}
G.~Baudat, F.~Anouar, Generalized discriminant analysis using a kernel
  approach, Neural Computation 12~(10) (2000) 2385--2404.

\bibitem{BachJ02}
F.~R. Bach, M.~I. I.~Jordan, Kernel independent component analysis, Journal of
  Machine Learning Research 3 (2002) 1--48.

\bibitem{BelkinNS06}
M.~Belkin, P.~Niyogi, V.~Sindhwani, Manifold regularization: {A} geometric
  framework for learning from labeled and unlabeled examples, Journal of
  Machine Learning Research 7 (2006) 2399--2434.

\bibitem{Aronszajn50}
N.~Aronszajn, Theory of reproducing kernels, Transactions of the American
  Mathematical Society 68~(3) (1950) 337--404.

\bibitem{DaiY07}
G.~Dai, D.~Y. Yeung, Kernel selection for semi-supervised kernel machines, in:
  Prof. 24th Int. Conf. Machine Learning, 2007, pp. 185--192.

\bibitem{ArgyriouHP05}
A.~Argyriou, M.~Herbster, M.~Pontil, Combining graph laplacians for
  semi-supervised learning, in: Advances in Neural Information Processing
  Systems 18, 2005, pp. 67--74.

\bibitem{TSMKL}
A.~Nazarpour, P.~Adibi, Two-stage multiple kernel learning for supervised
  dimensionality reduction, Pattern Recognition 48~(5) (2015) 1854--1862.

\bibitem{VajdaS16}
S.~Vajda, K.~C. Santosh, A fast k-nearest neighbor classifier using
  unsupervised clustering, in: Int. Conf. Recent Trends in Image Processing and
  Pattern Recognition, 2016, pp. 185--193.

\bibitem{Baxter92}
B.~J.~C. Baxter, The interpolation theory of radial basis functions, Ph.D.
  thesis, Cambridge University, Trinity College (1992).

\bibitem{Piret07}
C.~Piret, Analytical and numerical advances in radial basis functions, Ph.D.
  thesis, University of Colorado (2007).

\bibitem{GeBeKr01}
A.~S. Georghiades, P.~N. Belhumeur, D.~J. Kriegman, From few to many:
  Illumination cone models for face recognition under variable lighting and
  pose, IEEE Trans. Pattern Anal. Mach. Intelligence 23~(6) (2001) 643--660.

\bibitem{NeneNM96}
S.~A. Nene, S.~K. Nayar, H.~Murase, {C}olumbia {O}bject {I}mage {L}ibrary
  ({C}{O}{I}{L}-20), Tech. rep. (Feb 1996).

\bibitem{SamariaH94}
F.~Samaria, A.~Harter, Parameterisation of a stochastic model for human face
  identification, in: Proc. Second {IEEE} Workshop on Applications of Computer
  Vision, 1994, pp. 138--142.

\bibitem{ThomazG10}
C.~E. Thomaz, G.~A. Giraldi, A new ranking method for principal components
  analysis and its application to face image analysis, Image Vision Comput.
  28~(6) (2010) 902--913.

\bibitem{CSIEdata}
Robotics {CSIE} database for face detection, available:
  http://robotics.csie.ncku.edu.tw/Databases/FaceDetect\_PoseEstimate.htm.

\bibitem{MITCBCL}
{MIT-CBCL} face recognition database, available:
  http://cbcl.mit.edu/software-datasets/heisele/facerecognition-database.html.

\end{thebibliography}

\end{document}